\pdfoutput=1
\documentclass{article}

\PassOptionsToPackage{numbers, compress}{natbib}
\usepackage[preprint]{neurips_2026}

\usepackage{amsmath}
\usepackage{amsthm}
\usepackage{graphicx}
\usepackage{tcolorbox}
\usepackage{subcaption}
\usepackage{algorithm}
\usepackage{algorithmicx}
\usepackage{algcompatible}
\usepackage{algpseudocode}

\usepackage[utf8]{inputenc} 
\usepackage[T1]{fontenc}    
\usepackage{hyperref}       
\usepackage{url}            
\usepackage{booktabs}       
\usepackage{amsfonts}       
\usepackage{nicefrac}       
\usepackage{microtype}      
\usepackage{xcolor}         

\newtheorem{theorem}{Theorem}
\newtheorem{proposition}{Proposition}
\newtheorem{lemma}{Lemma}
\newtheorem{corollary}{Corollary}
\title{Predictive Geometry of Hidden Trajectories in Transformers}

\author{%
  Timur Mudarisov$^{1}$ \quad
  Mikhail Burtsev$^{2}$ \quad
  Tatiana Petrova$^{1}$ \quad
  Radu State$^{1}$ \\[0.6em]
  \normalfont $^{1}$University of Luxembourg, Luxembourg \\
  \normalfont $^{2}$London Institute for Mathematical Sciences, London, UK
}

\begin{document}

\maketitle

\begin{abstract}
Decoder-only transformers are trained only through a terminal next-token prediction loss, yet this loss constrains every intermediate hidden state through the fixed downstream computation. We formalize this constraint by studying layerwise loss-to-go functions: the terminal loss obtained by continuing a candidate hidden state through the remaining transformer blocks. Around successful validation trajectories, we show that the local second-order geometry of these functions is governed, up to low-loss residual terms, by a pullback Fisher operator on hidden-state space. Its spectrum identifies output-sensitive directions and approximately prediction-null directions, yielding a local observable subspace of the residual stream. For causal transformers, the same geometry induces a tokenwise curvature score: a Fisher-weighted sensitivity of the target logits to perturbations of each token's hidden state. This score vanishes outside the causal ancestor set of the target and is controlled by downstream Jacobian couplings, making it a loss-aware alternative to attention magnitude. We estimate these quantities using matrix-free Jacobian-vector and vector-Jacobian products and evaluate them across decoder-only language models on WikiText, OpenWebText, and FineWeb. Empirically, the induced geometry predicts perturbation sensitivity, supports nonuniform layerwise rank allocation, yields competitive structured token-pruning signals, and improves low-rank student recovery when added to stronger autoregressive distillation objectives such as reverse KL and skew KL. These results support a predictive-geometric view of transformer computation: near successful trajectories, the terminal loss induces a thin, anisotropic set of output-relevant hidden-state directions that can be measured and exploited for compression and distillation.
\end{abstract}

\section{Introduction}

Large-scale decoder-only transformers \citep{vaswani2017attention, brown2020gpt3}
are optimized through a next-token prediction loss applied only at the output
layer, yet their computation proceeds through a sequence of intermediate hidden
states. This raises a local question about trained models: near hidden-state
trajectories on which the model already predicts well, which directions in
residual-stream space are constrained by the terminal prediction loss, and which
directions are approximately invisible to it? Answering this question would give
a loss-based account of hidden-state geometry, rather than treating intermediate representation geometry as an empirical
property detached from the objective used to train the model.

Recent empirical work suggests that transformer hidden states are structured and
compressible rather than generic high-dimensional vectors. Prior work
\citep{doimo2023geometry} shows that the intrinsic dimension of hidden
representations follows a characteristic expansion-contraction profile across
layers, while activation-space compression methods
\citep{ashkboos2024slicegpt, tian2026flatlm, yuan2023asvd} demonstrate that low-rank
approximations can often preserve model performance. These results indicate that
hidden representations contain substantial redundancy. However, they largely
characterize this structure empirically: they do not derive the relevant
directions from the terminal prediction objective itself, nor do they provide a
loss-aware notion of which hidden-state directions or token positions matter for
the final prediction.

In this work, we start from the prediction loss. For a fixed trained
decoder-only transformer, a layer, and a target position, we define a
\emph{loss-to-go} function: the terminal loss obtained by continuing a candidate
intermediate hidden state through the fixed downstream transformer. We then study
the local second-order geometry of this function around successful validation
trajectories. In the low-loss regime, the dominant curvature is given by a
pullback Fisher operator: the Fisher geometry of the output distribution pulled
back through the downstream computation to the intermediate hidden state. This
operator identifies output-sensitive directions in hidden-state space together
with approximately prediction-null directions that leave the centered target
logits unchanged to first order.

This local output-induced geometry leads to directly measurable quantities. The
spectrum of the pullback Fisher operator measures how many output-sensitive
directions are present at each layer, suggesting a data-dependent alternative to
uniform rank allocation. Taking traces of this operator over token fibers defines a local token-level
sensitivity score, measuring how much each token position contributes to the
curvature of the target prediction. For causal transformers, this score is
supported only on causal ancestors of the target token and is controlled by
downstream Jacobian couplings, making it a loss-aware alternative to
attention-based pruning heuristics
\citep{zerotprune2023, tang2023top, yun2023focus}. We use these quantities as
signals for matched-budget rank allocation and structured token pruning.

Our perspective connects representation geometry
\citep{doimo2023geometry}, compression and low-rank adaptation
\citep{ashkboos2024slicegpt, hu2022lora, hayou2024loraplus, tian2026flatlm},
token pruning \citep{zerotprune2023, tang2023top, yun2023focus},
mechanistic interpretability \citep{elhage2021mathematical}, and Fisher-based
analyses of neural models \citep{amari1998natural, karakida2020fisher}. The
focus here is specifically on the \emph{local predictive geometry} of
intermediate hidden states: the Fisher geometry induced by the terminal
next-token loss and pulled back through the learned forward dynamics of a fixed
trained transformer.
This paper makes the following contributions:

\begin{enumerate}
\item \textbf{Loss-induced hidden geometry.}
We define layerwise loss-to-go functions for fixed trained decoder-only transformers
and show that, near successful hidden-state trajectories, their dominant local
curvature is a pullback Fisher operator from the output distribution to the
intermediate hidden state.

\item\textbf{Observable and prediction-null directions.}
The pullback Fisher operator identifies output-sensitive directions and
approximately prediction-null directions in residual-stream space, yielding a
testable local decomposition of hidden-state perturbations.

\item\textbf{Tokenwise predictive saliency.}
For causal transformers, we derive a tokenwise Fisher-Jacobian score
$\kappa_{\ell,s}$ that measures how much each token position contributes to
target-logit curvature. The score is supported only on causal ancestors of the
target and is controlled by downstream Jacobian couplings.

\item\textbf{Geometry-guided reduction and recovery.}
We evaluate the resulting geometry across perturbation sensitivity, layerwise
rank allocation, structured token pruning, and low-rank student distillation.
Across these settings, the geometry provides loss-aware signals for preserving
prediction-relevant hidden directions under compression.
\end{enumerate}

\section{Problem Statement}

We study a fixed trained decoder-only pre-norm transformer at a fixed target
position \(t\). For an input prefix \(s_{\le t}\), let
\[
X_\ell(s_{\le t}) \in \mathcal X_t := \mathbb R^{t \times d},
\qquad \ell=0,\dots,L,
\]
denote the layer-\(\ell\) hidden state, where \(d\) is the hidden dimension.
The model induces a hidden trajectory
\[
X_0(s_{\le t}) \to X_1(s_{\le t}) \to \cdots \to X_L(s_{\le t}),
\]
and the final prediction at position \(t\) is read out from \(X_L\). Writing
the layer dynamics as
\[
X_{\ell+1}=B_\ell(X_\ell),
\qquad
\Phi_{\ell\to L}:=B_{L-1}\circ\cdots\circ B_\ell,
\]
we analyze the geometry induced by the learned forward dynamics of a
\emph{fixed trained model}, rather than the training dynamics themselves.

To measure prediction quality at position \(t\), let \(\phi\) be a terminal
loss applied after the final readout, such as cross-entropy or
\[
\phi(X_L):=D_{\mathrm{KL}}\!\bigl(q \,\|\, p_t(\cdot\mid X_L)\bigr),
\]
where \(p_t(\cdot\mid X_L)\) is the prediction induced by the final hidden state
\(X_L\). Define the layerwise loss-to-go
\[
\mathcal J_\ell(X):=\phi\!\bigl(\Phi_{\ell\to L}(X)\bigr),
\]
and the associated low-loss set
\[
\mathcal A_\varepsilon^\ell
:=
\{\,X\in\mathcal X_t : \mathcal J_\ell(X)\le \varepsilon\,\}.
\]
Thus, \(\mathcal A_\varepsilon^\ell\) consists of those intermediate states at
layer \(\ell\) that can still evolve, under the fixed downstream transformer, to
terminal loss at most \(\varepsilon\).

Our goal is to characterize the geometry of these layerwise low-loss sets and
its consequences for transformer computation. In particular, we ask whether
successful trajectories remain inside a pullback family of low-loss regions,
whether the local geometry of these regions is effectively lower-dimensional
than the ambient residual space \(\mathcal X_t\), and how this
prediction-relevant geometry is distributed across token positions through the
causal computation graph. Answering these questions gives both a
trajectory-level description of transformer computation beyond the final logits
and a basis for loss-aware model-reduction signals, including layerwise
low-rank compression and token-sparse computation.

\section{Theory}
\label{sec:theory}

We study the local geometry of intermediate hidden states that remain compatible
with low terminal prediction loss. Throughout, we fix a trained decoder-only
pre-norm transformer, a target position \(t\), and a softmax readout, with
terminal loss
\[
\phi(z):=D_{\mathrm{KL}}\!\bigl(q\,\|\,\operatorname{softmax}(z)\bigr),
\qquad
z\in\mathbb R^{|\mathcal V|}.
\]
Proofs are presented in the Appendix\,\ref{ap:theory}.

\subsection{Local geometry of layerwise low-loss sets}

Let
\[
X_{\ell+1}=B_\ell(X_\ell),
\qquad
\Phi_{\ell\to L}:=B_{L-1}\circ\cdots\circ B_\ell,
\qquad
\Psi_\ell(X):=Z_t\!\bigl(\Phi_{\ell\to L}(X)\bigr).
\]
Define the layerwise loss-to-go and low-loss set by
\begin{equation}
\mathcal J_\ell(X):=\phi\!\bigl(\Psi_\ell(X)\bigr),
\qquad
\mathcal A_\varepsilon^\ell
:=
\{X\in\mathcal X_t:\mathcal J_\ell(X)\le\varepsilon\}.
\end{equation}
Then
\begin{equation}
\mathcal A_\varepsilon^\ell
=
\Phi_{\ell\to L}^{-1}(\mathcal A_\varepsilon^L),
\qquad
\mathcal A_\varepsilon^L
=
\{X\in\mathcal X_t:\phi(Z_t(X))\le\varepsilon\},
\end{equation}
so the main question is the \emph{local geometry} of these sets near a
successful hidden-state trajectory.

\begin{theorem}[Local output geometry of low-loss sets]
\label{th1}
Let \(X_\ell^\ast\in\mathcal X_t\) be a reference hidden state with
\[
\varepsilon_\ast:=\mathcal J_\ell(X_\ell^\ast),
\qquad
\mathcal J_\ell(X):=\phi(\Psi_\ell(X)),
\qquad
\phi(z):=D_{\mathrm{KL}}\!\left(q\,\|\,\operatorname{softmax}(z)\right).
\]
Assume that \(\Psi_\ell\) is twice continuously differentiable in a
neighborhood of \(X_\ell^\ast\) and that
\[
\left(
\sum_{a=1}^{|\mathcal V|}
\|\nabla^2\Psi_{\ell,a}(X)\|_{\mathrm{op}}^2
\right)^{1/2}
\le M_\ell
\]
in that neighborhood. For transformers with smooth activations and LayerNorm, this condition holds locally on neighborhoods where LayerNorm denominators are bounded away from zero and attention probabilities remain finite.
Let
\[
P:=I-\tfrac{1}{|\mathcal V|}\mathbf 1\mathbf 1^\top,
\qquad
\bar\Psi_\ell:=P\Psi_\ell,
\qquad
\bar J_\ell:=D\bar\Psi_\ell(X_\ell^\ast),
\]
and let
\[
p^\ast:=\operatorname{softmax}(\Psi_\ell(X_\ell^\ast)),
\qquad
F^\ast:=\operatorname{diag}(p^\ast)-p^\ast p^{\ast\top}.
\]
Then:

\begin{enumerate}
    \item \textnormal{\emph{Local Hessian decomposition.}}
    The Hessian of the layerwise loss-to-go satisfies
    \begin{equation}
    \label{eq:hessian_decomp}
    \nabla^2\mathcal J_\ell(X_\ell^\ast)
    =
    \bar J_\ell^\top F^\ast \bar J_\ell
    +
    R_\ell,
    \end{equation}
    where
    \begin{equation}
    \label{eq:remainder_bound}
    \|R_\ell\|_{\mathrm{op}}
    \le
    M_\ell\|p^\ast-q\|_2
    \le
    M_\ell\sqrt{2\varepsilon_\ast}.
    \end{equation}
    Thus, at low loss, the dominant curvature term is the pullback of the
    Fisher geometry of the output distribution.

    \item \textnormal{\emph{Pullback Fisher geometry.}}
    Define
    \begin{equation}
    \label{eq:Kl_def}
    K_\ell:=\bar J_\ell^\top F^\ast \bar J_\ell,
    \qquad
    g_\ell(\delta X,\delta X):=
    \langle \delta X,K_\ell\delta X\rangle .
    \end{equation}
    For \(X=X_\ell^\ast+\delta X\) sufficiently close to \(X_\ell^\ast\),
    \begin{equation}
    \label{eq:local_taylor}
    \mathcal J_\ell(X)
    =
    \mathcal J_\ell(X_\ell^\ast)
    +
    \langle\nabla\mathcal J_\ell(X_\ell^\ast),\delta X\rangle
    +
    \frac12 g_\ell(\delta X,\delta X)
    +
    O(\sqrt{\varepsilon_\ast}\|\delta X\|^2)
    +
    O(\|\delta X\|^3).
    \end{equation}
    Hence, near a low-loss point with small first-order term, the local
    low-loss set is approximated by an ellipsoidal tube.

    \item \textnormal{\emph{Prediction-null directions.}}
    If \(\bar J_\ell\) has constant rank \(r\) near \(X_\ell^\ast\), then
    \begin{equation}
    \mathcal N_\ell:=\ker\bar J_\ell
    \end{equation}
    consists of first-order prediction-null directions. 
\end{enumerate}
\end{theorem}
\begin{tcolorbox}[colback=gray!5!white, colframe=gray!40!black,
                  title={\small Informal summary of Theorem~\ref{th1}},
                  left=4pt, right=4pt, top=3pt, bottom=3pt]
\small
A successful hidden-state trajectory stays inside a ``tube'' of low-loss states.
This tube is thin in prediction-relevant directions and thick in null
directions; removing the latter leaves a smaller observable space whose local
geometry is described by \(g_\ell\).
\end{tcolorbox}

\begin{figure}[t]
    \centering
    \includegraphics[width=0.7\linewidth]{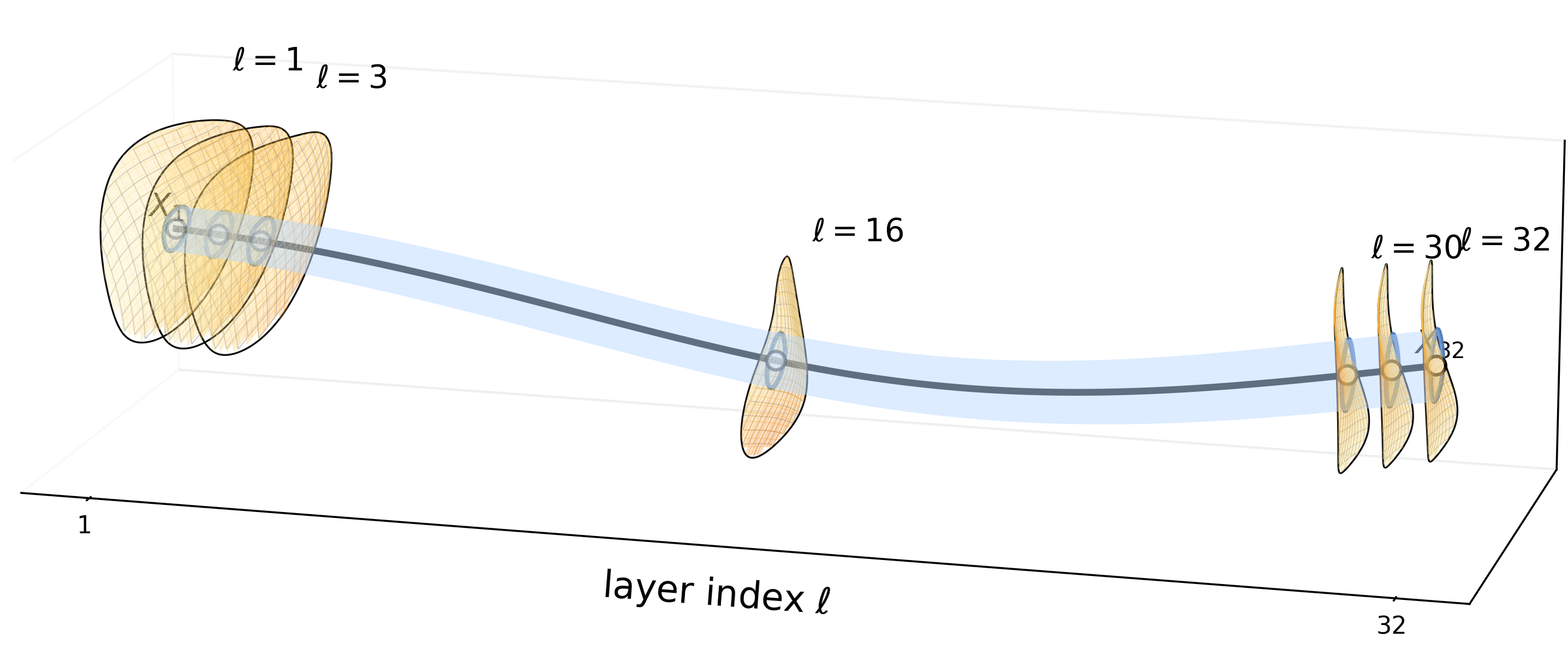}
    \vspace{0.5em}
    \begin{minipage}{0.4\linewidth}
        \centering
        \includegraphics[width=0.8\linewidth]{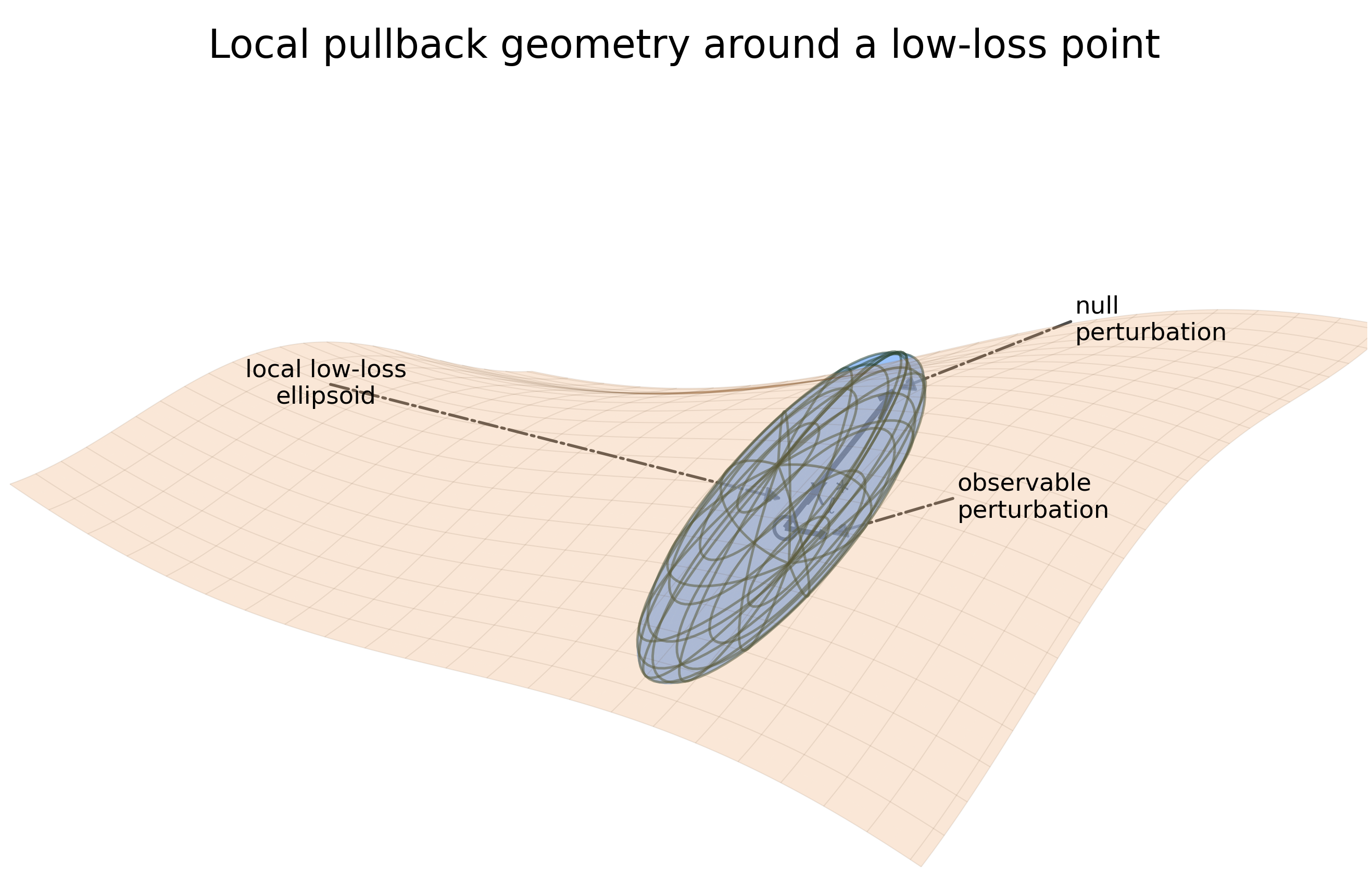}
    \end{minipage}
    \hfill
    \begin{minipage}{0.55\linewidth}
        \centering
        \includegraphics[width=0.85\linewidth]{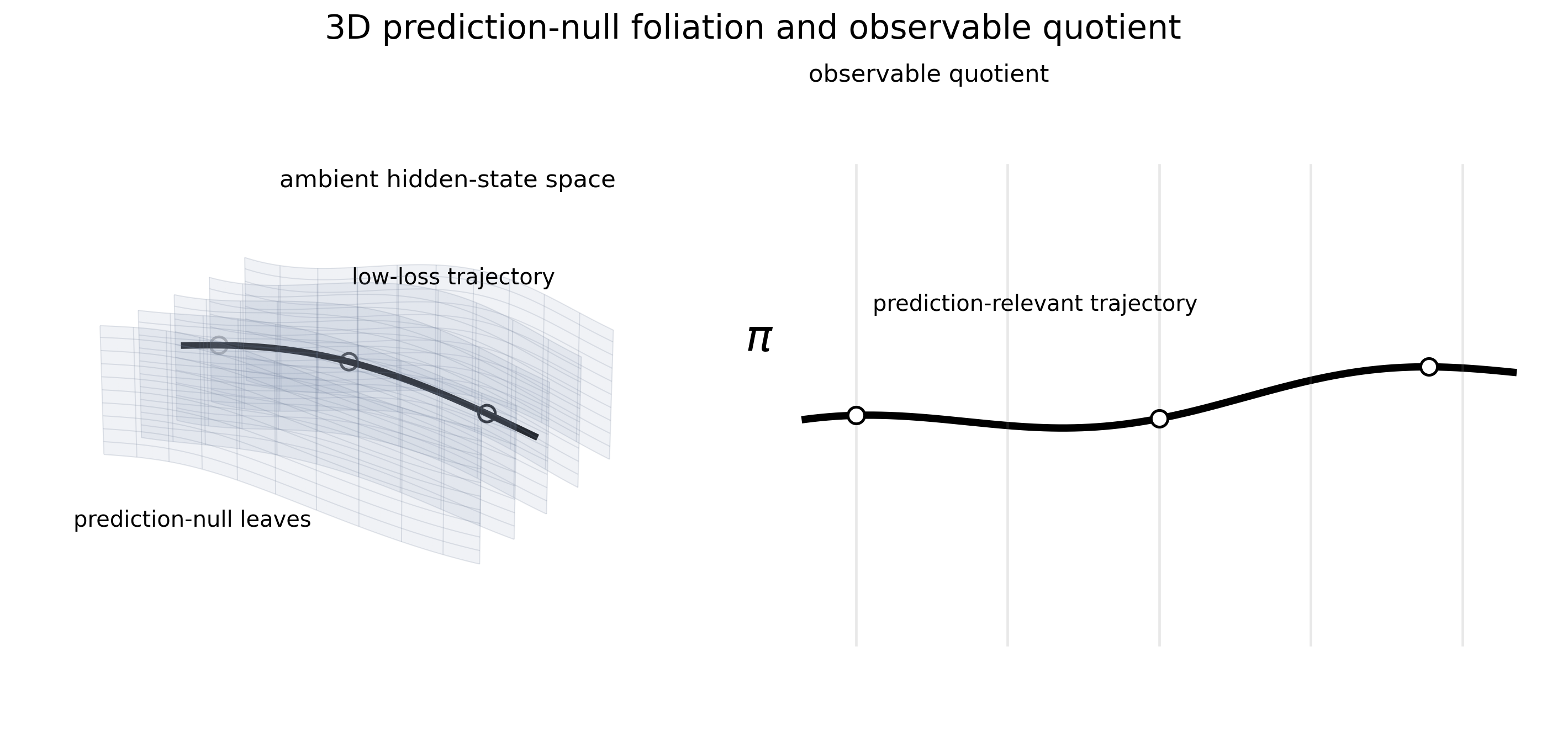}
    \end{minipage}
    \caption{\textbf{Geometry of low-loss hidden trajectories.}
    \textbf{Top:} hidden-state trajectory and layerwise low-loss sets
    \(\mathcal A_\varepsilon^\ell\).
    \textbf{Bottom left:} local pullback geometry near \(X_\ell^\ast\).
    \textbf{Bottom right:} prediction-null foliation and observable quotient.}
\end{figure}

\subsection{Tokenwise geometric saliency in transformers}

Let
\[
G_\ell:=DB_\ell(X_\ell^\ast),
\qquad
G_\ell^{(r\leftarrow s)}:\mathbb R^d\to\mathbb R^d
\]
denote the Jacobian of the \(\ell\)-th block in token blocks. The operator
\(K_\ell\) in \eqref{eq:Kl_def} is the local output-sensitive curvature
operator. For each token \(s\), let
\[
P_s:\mathcal X_t\to \mathbb R^d
\]
project onto the \(s\)-th token fiber, and define the geometric contribution
score
\begin{equation}
\kappa_{\ell,s}:=\operatorname{tr}(P_sK_\ell P_s^\top).
\end{equation}
Thus \(\kappa_{\ell,s}\) measures how much of the local output-sensitive
curvature is carried by token \(s\).

For each token pair \((s,r)\), define
\[
w_\ell^{(r\leftarrow s)}:=\|G_\ell^{(r\leftarrow s)}\|_F.
\]
These weights induce a causal directed token graph at layer \(\ell\). Let
\(\mathrm{Anc}_{\ell:L-1}(t)\) be the ancestor closure of the target token \(t\)
under the composed layerwise graphs.

\begin{proposition}[Tokenwise pullback saliency and causal support]
\label{pr1}
Let \(K_\ell=\bar J_\ell^\top F^\ast \bar J_\ell\) be the pullback Fisher
operator from Theorem~\ref{th1}. For a token position \(s\le t\), let
\(P_s:\mathcal X_t\to \mathbb R^d\) denote projection onto the \(s\)-th token
fiber, and define the downstream token-to-logit Jacobian
\[
H_{\ell,s}
:=
D_{X_{\ell,s}}\bar\Psi_\ell(X_\ell^\ast)
=
\bar J_\ell P_s^\top
\in \mathbb R^{|\mathcal V|\times d}.
\]
Define the tokenwise curvature score
\[
\kappa_{\ell,s}
:=
\operatorname{tr}(P_sK_\ell P_s^\top).
\]
Then
\begin{enumerate}
    \item \textnormal{\emph{Fisher-Jacobian identity.}}
    \begin{equation}
    \label{eq:kappa_identity}
    \kappa_{\ell,s}
    =
    \operatorname{tr}\!\left(H_{\ell,s}^\top F^\ast H_{\ell,s}\right)
    =
    \left\|F^{\ast 1/2}H_{\ell,s}\right\|_F^2 .
    \end{equation}
    Thus \(\kappa_{\ell,s}\) is the Fisher-weighted squared sensitivity of the
    target-position logits to perturbations of token \(s\) at layer \(\ell\).

    \item \textnormal{\emph{Causal support.}}
    Under the standard causal mask, if token \(s\) is not in the downstream
    ancestor set of target token \(t\) from layer \(\ell\) to \(L\), then
    \[
    H_{\ell,s}=0,
    \qquad
    \kappa_{\ell,s}=0.
    \]

    \item \textnormal{\emph{Pathwise coupling bound.}}
    Let \(G_m^{(r\leftarrow u)}\) denote the token-block Jacobian of block
    \(m\), mapping perturbations at token \(u\) to perturbations at token \(r\).
    Let \(\Pi_{\ell:s\to t}\) be the set of directed causal paths from
    \((\ell,s)\) to the target token \(t\) at layer \(L\). Then
    \begin{equation}
    \label{eq:kappa_path_bound}
    \kappa_{\ell,s}
    \le
    d\|F^\ast\|_{\mathrm{op}}
    \|D Z_t(X_L^\ast)\|_{\mathrm{op}}^2
    \left(
    \sum_{\pi\in\Pi_{\ell:s\to t}}
    \prod_{(m,u\to r)\in\pi}
    \|G_m^{(r\leftarrow u)}\|_{\mathrm{op}}
    \right)^2 .
    \end{equation}
\end{enumerate}
\end{proposition}

\begin{tcolorbox}[colback=gray!5!white, colframe=gray!40!black,
                  title={\small Informal summary of Proposition~\ref{pr1}},
                  left=4pt, right=4pt, top=3pt, bottom=3pt]
\small
A token contributes to the output geometry only if it causally influences the
target. Even among such ancestors, its contribution is small when its Jacobian
coupling to the target is weak. Geometric importance is therefore both sparse
and graded.
\end{tcolorbox}

\subsection{Geometry-guided reduction}

The spectrum of \(K_\ell\) identifies prediction-relevant directions, while
Proposition~\ref{pr1} shows that this geometry is often concentrated on a small
subset of ancestor tokens. This suggests projecting onto the dominant eigenspace
of \(K_\ell\) and pruning tokens with small \(\kappa_{\ell,s}\).

\begin{proposition}[Geometry-guided reduction with local loss guarantee]
\label{pr2}
Let \(K_\ell\) be the pullback Fisher operator from Theorem~\ref{th1}, with
eigenvalues
\[
\lambda_1(K_\ell)\ge \lambda_2(K_\ell)\ge \cdots \ge 0.
\]
Let \(\Pi_\ell^{(k)}\) be the orthogonal projector onto the span of the top
\(k\) eigenvectors of \(K_\ell\). For a local hidden-state perturbation
\[
X=X_\ell^\ast+\delta X,
\]
define its geometry-guided reduction by
\[
\widetilde X
:=
X_\ell^\ast+\Pi_\ell^{(k)}\delta X.
\]
Assume that \(X\) and \(\widetilde X\) lie in the neighborhood where
Theorem~\ref{th1} applies. Then the discarded component
\[
\delta X_{\mathrm{tail}}
:=
(I-\Pi_\ell^{(k)})\delta X
\]
has local quadratic loss contribution bounded by
\begin{equation}
\label{eq:geometry_reduction_tail_bound}
\frac12
\left\langle
\delta X_{\mathrm{tail}},
K_\ell
\delta X_{\mathrm{tail}}
\right\rangle
\le
\frac12
\lambda_{k+1}(K_\ell)
\|\delta X_{\mathrm{tail}}\|^2 .
\end{equation}

\end{proposition}

\section{Empirical Validation and Geometry-Guided Reduction}
\label{sec:experiments}

The theory yields two kinds of objects. The first are \emph{descriptive}
observables of the hidden trajectory, such as the operator \(K_\ell\) and the
tokenwise score \(\kappa_{\ell,s}\), which can be measured directly on trained
models and used to test the predictions of Theorem~\ref{th1} and
Proposition~\ref{pr1}. The second are \emph{constructive} objects, such as
layerwise observable subspaces and tokenwise saliency rankings, which can be used to design reduced architectures and geometry-aware compression schemes.

This section evaluates the theory at two levels. First, we test whether the
pullback-Fisher geometry is a faithful local description of hidden-state
sensitivity: observable directions should be substantially more output-sensitive
than prediction-null directions. Second, we test whether the same geometry is
useful as an algorithmic signal for compression and recovery. This section evaluates the theory in five linked steps. 
\textbf{Experiment~\ref{sec:exp0-local-taylor}} directly validates the local Taylor geometry induced by the pullback Fisher operator. 
\textbf{Experiment~\ref{sec:exp1}} tests the observable/null anisotropy and effective observable dimension predicted by the low-loss tube picture. 
\textbf{Experiment~\ref{sec:exp2}} uses the spectrum of \(K_\ell\) for matched-budget rank allocation. 
\textbf{Experiment~\ref{sec:exp3}} uses the token-fiber trace \(\kappa_{\ell,s}\) for structured token-update pruning. \textbf{Experiment~\ref{sec:pg-distillation}} uses the same geometry as a hidden-state recovery objective for low-rank student distillation.

Unless stated otherwise, all measurements are computed on held-out validation sequences from \textsc{WikiText}~\citep{merity2016pointer}, \textsc{FineWeb}~\citep{penedo2024fineweb}, and \textsc{OpenWebText}~\citep{gokaslan2019openwebtext}.

\subsection{Experiment 0: Direct validation of local predictive geometry}
\label{sec:exp0-local-taylor}

Before using the pullback Fisher operator for reduction or recovery, we first test whether it captures the local predictive geometry predicted by Theorem~\ref{th1}. Specifically, we ask whether small perturbations of an intermediate hidden state produce terminal loss changes predicted by the local Taylor model induced by \(K_\ell\). For supervised cross-entropy, we compare a linear predictor, a Fisher quadratic predictor, and their sum:
\(\Delta \mathrm{CE}_{\mathrm{lin}}=\langle \nabla_{X_\ell}\mathrm{CE},\delta X_\ell\rangle\),
\(\Delta \mathrm{CE}_{\mathrm{quad}}=\frac{1}{2}\delta X_\ell^\top K_\ell\delta X_\ell\), and
\(\Delta \mathrm{CE}_{\mathrm{lin+quad}}=\Delta \mathrm{CE}_{\mathrm{lin}}+\Delta \mathrm{CE}_{\mathrm{quad}}\).
This separation is necessary because supervised one-hot CE is generally not stationary at the reference prediction, so the linear term need not vanish. To isolate the Fisher curvature term, we also evaluate teacher-KL with \(q=p^\star\), the model's unperturbed prediction. In this stationary setting, the leading local term is
\(\Delta \mathrm{KL}(p^\star\|p_\delta)\approx \frac{1}{2}\delta X_\ell^\top K_\ell\delta X_\ell\).

Figure~\ref{fig:exp1-r2-by-radius} shows that supervised CE is first-order dominated: the linear-only and linear-plus-quadratic predictors nearly coincide, while the quadratic-only predictor is substantially weaker. In contrast, the teacher-KL quadratic predictor becomes strongly predictive at moderate radii, where the perturbation-induced KL signal is measurable. Thus, Experiment~0 validates \(K_\ell\) as a local output-curvature or observability operator. The following experiments then use this geometry constructively: Experiment~1 studies observable and prediction-null directions, Experiment~2 uses the spectrum for rank allocation, Experiment~3 uses token-fiber traces for pruning, and Experiment~4 uses the quadratic form as a hidden-state recovery objective.

\subsection{Experiment 1: anisotropy and effective observable dimension}
\label{sec:exp1}

Guided by Theorem~\ref{th1}, we quantify how strongly each layer reacts to
perturbations in observable versus prediction-null directions. For every model,
layer \(\ell\), and target position \(t\), we construct the local curvature
operator \(K_\ell\), compute its top eigendirections, and use them to define an
observable subspace together with its orthogonal complement. We then sample
equal-norm perturbations in each subspace, apply them to \(X_\ell(s)\), and
measure the resulting change in target-position KL divergence between the
original and perturbed outputs.

To summarize the spectrum of \(K_\ell\), we report the effective observable
rank via the participation ratio
\begin{equation}
r_{\mathrm{eff}}(K_\ell)
:=
\frac{\bigl(\sum_i \lambda_i(K_\ell)\bigr)^2}
{\sum_i \lambda_i(K_\ell)^2}.
\end{equation}
This quantity serves as an empirical proxy for the low-dimensional
prediction-relevant structure suggested by Theorem~\ref{th1}. Figure~\ref{fig:exp1_geometry} shows a clear and consistent anisotropy across
all models: observable perturbations are orders of magnitude more harmful than
null perturbations at almost every layer, and the gap typically widens toward
the later blocks. This indicates that the residual geometry becomes
increasingly aligned with the final readout as the computation approaches the
output layer. The right panel further shows that the effective observable rank
grows gradually with depth, especially in the middle and later layers, but
remains far below the ambient hidden dimension. Taken together, these trends
support the view that decoder-only transformers concentrate most of their
prediction-relevant information in a thin, highly anisotropic subset of hidden
space, while a large fraction of directions remain nearly KL-neutral.

\begin{figure}[t]
    \centering
    \begin{minipage}{0.49\linewidth}
        \centering
        \includegraphics[width=.8\linewidth]{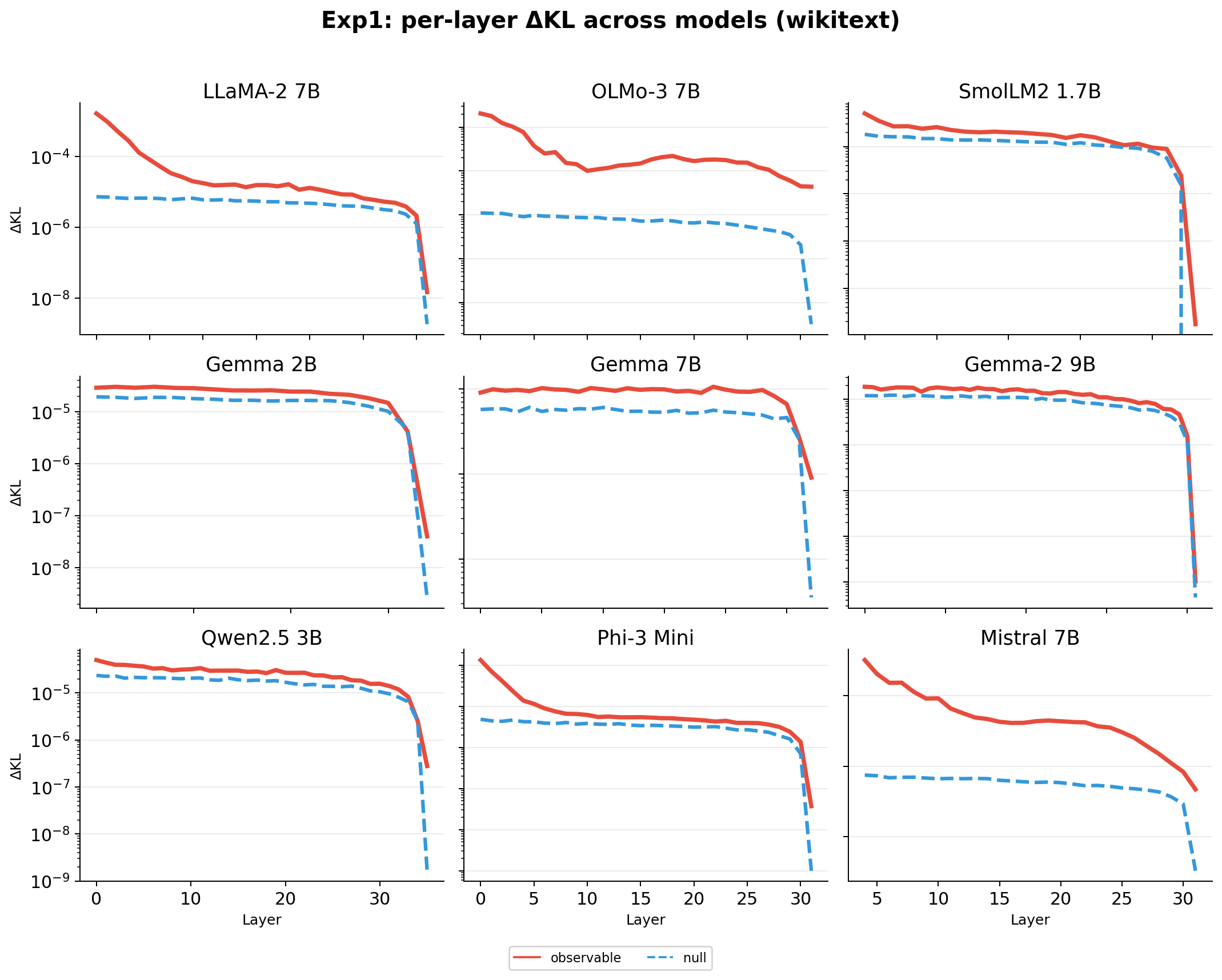}
    \end{minipage}
    \hfill
    \begin{minipage}{0.49\linewidth}
        \centering
        \includegraphics[width=.8\linewidth]{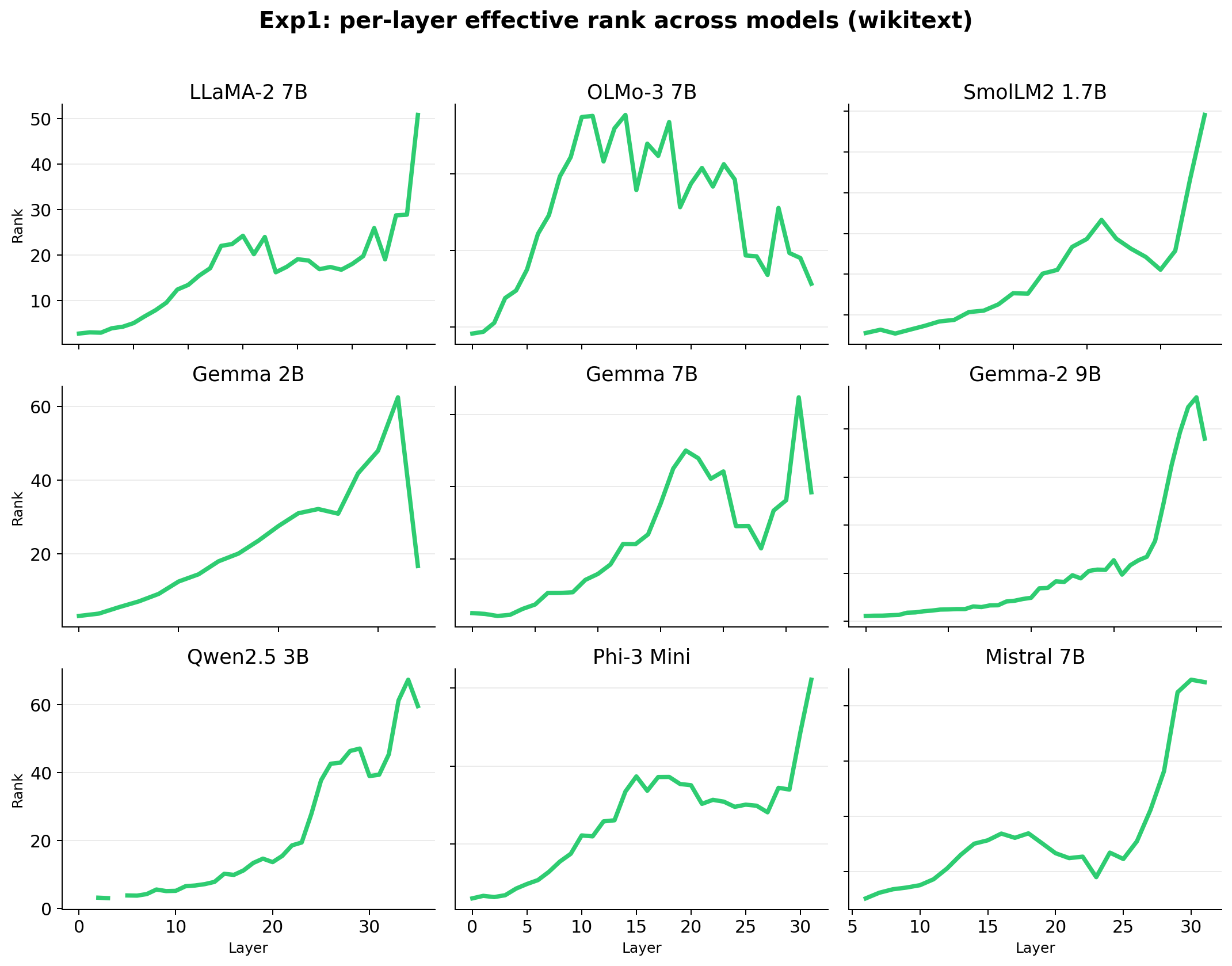}
    \end{minipage}
    \caption{\textbf{Experiment 1: anisotropy and effective observable dimension across layers.}
    \textbf{Left:} KL sensitivity of equal-norm observable and null perturbations across layers.
    \textbf{Right:} Effective observable rank across layers for multiple models.}
    \label{fig:exp1_geometry}
\end{figure}

\subsection{Experiment 2: geometry-guided rank allocation}
\label{sec:exp2}

Experiment~2 asks whether the layerwise geometry can be used to allocate low
rank more effectively than uniform compression. Given a global compression
budget, our geometry-guided method assigns a data-dependent rank \(k_\ell\) to
each layer using the spectrum of \(K_\ell\). Concretely, for a threshold
\(\tau\in\{0.01,0.05,0.1,0.2\}\), we choose the smallest \(k_\ell\) such that
\begin{equation}
\sum_{i>k_\ell}\lambda_i(K_\ell)
\;\le\;
\tau \sum_i \lambda_i(K_\ell).
\end{equation}
Intuitively, \(\tau\) controls how much spectral tail mass is discarded. As controlled matched-budget baselines, we use uniform SVD and uniform LoRA-style low-rank allocations~\citep{hu2022lora,hayou2024loraplus}, assigning the same rank to every layer and matching the overall compression ratio.

For each model, dataset, and threshold \(\tau\), we report the perplexity
increase
\begin{equation}
\Delta\mathrm{PPL}
=
\mathrm{PPL}_{\mathrm{compressed}}
-
\mathrm{PPL}_{\mathrm{base}}
\end{equation}
and the corresponding advantage over the best matched uniform baseline,
\begin{equation}
\Delta\mathrm{PPL}_{\mathrm{adv}}
:=
\Delta\mathrm{PPL}_{\mathrm{uniform}}
-
\Delta\mathrm{PPL}_{\mathrm{geom}}.
\end{equation}
Larger \(\Delta\mathrm{PPL}_{\mathrm{adv}}\) therefore indicates a stronger
benefit from geometry-guided allocation. We evaluate ten decoder-only LLMs
between roughly \(1\)B and \(9\)B parameters (see Table\,\ref{tab:model_checkpoints}). 

Figure~\ref{fig:exp2_advantage_wikitext_all_tau} shows that geometry-guided rank
allocation is most helpful at moderate thresholds (\(\tau=0.05\) and
\(\tau=0.1\)), where several small and medium models, including Gemma 2B,
Qwen2.5 3B, SmolLM2 1.7B, and OLMo-3 7B, benefit from a nonuniform rank
distribution. This suggests that these architectures contain a mix of strongly
prediction-relevant and comparatively redundant layers. For larger and more
homogeneous model families, especially some LLaMA and Gemma variants, the
advantage is smaller, indicating that uniform low-rank compression is already
close to optimal when layerwise spectra look similar. 

\begin{figure}[t]
  \centering
  \begin{subfigure}{0.48\linewidth}
    \centering
    \includegraphics[width=0.8\linewidth]{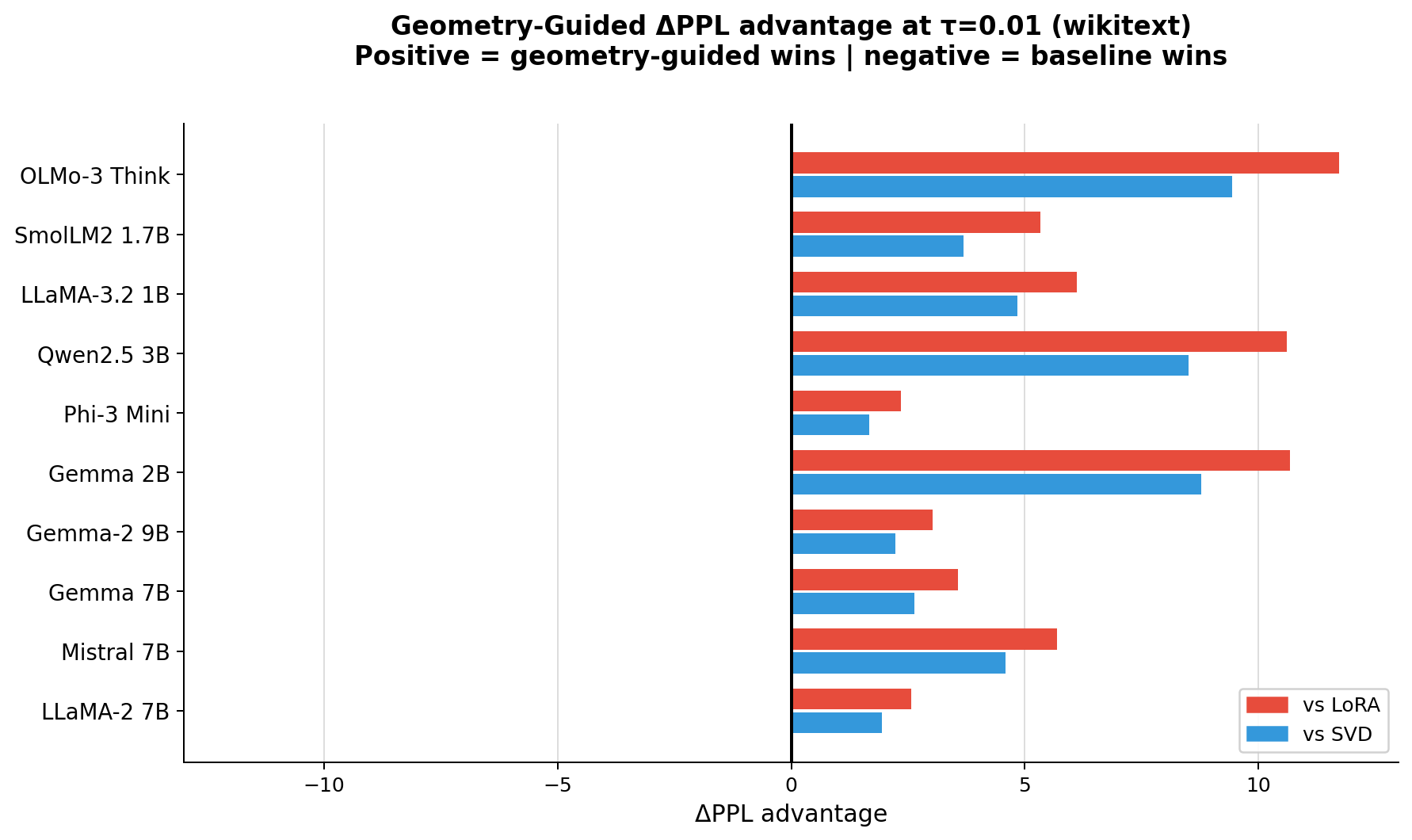}
    \label{fig:exp2_tau001}
  \end{subfigure}\hfill
  \begin{subfigure}{0.48\linewidth}
    \centering
    \includegraphics[width=0.8\linewidth]{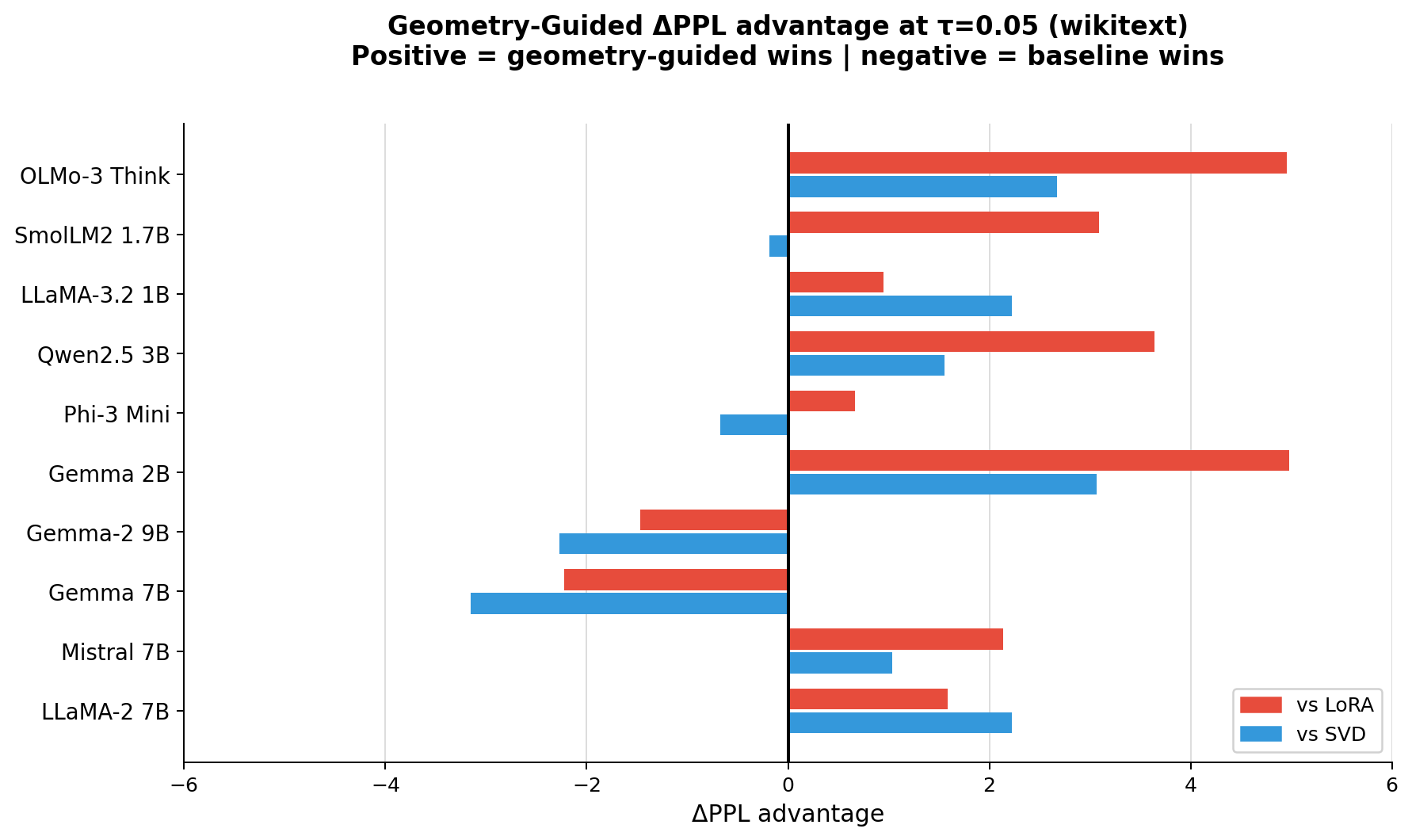}
    \label{fig:exp2_tau005}
  \end{subfigure}

  \vspace{0.5em}

  \begin{subfigure}{0.48\linewidth}
    \centering
    \includegraphics[width=0.8\linewidth]{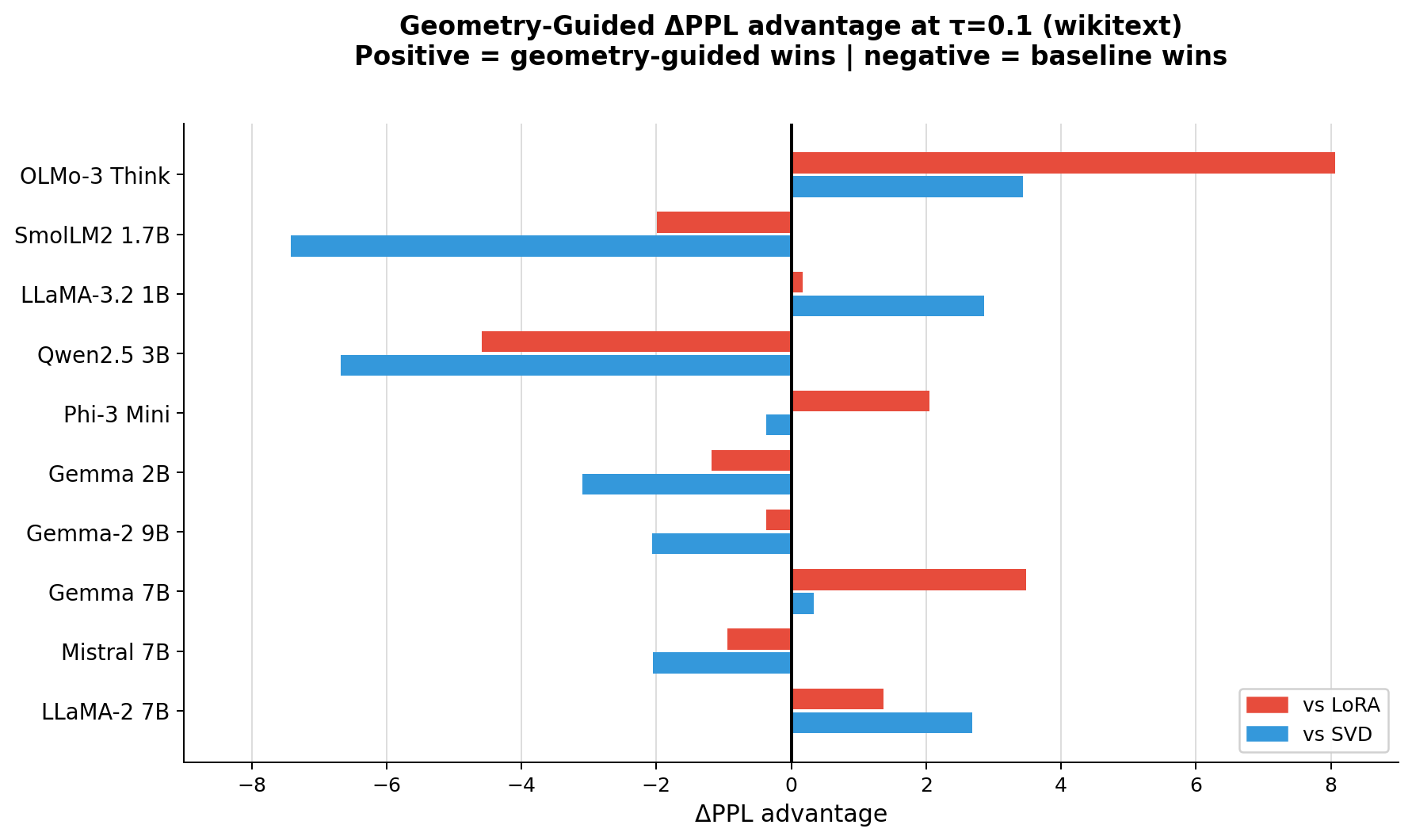}
    \label{fig:exp2_tau01}
  \end{subfigure}\hfill
  \begin{subfigure}{0.48\linewidth}
    \centering
    \includegraphics[width=0.8\linewidth]{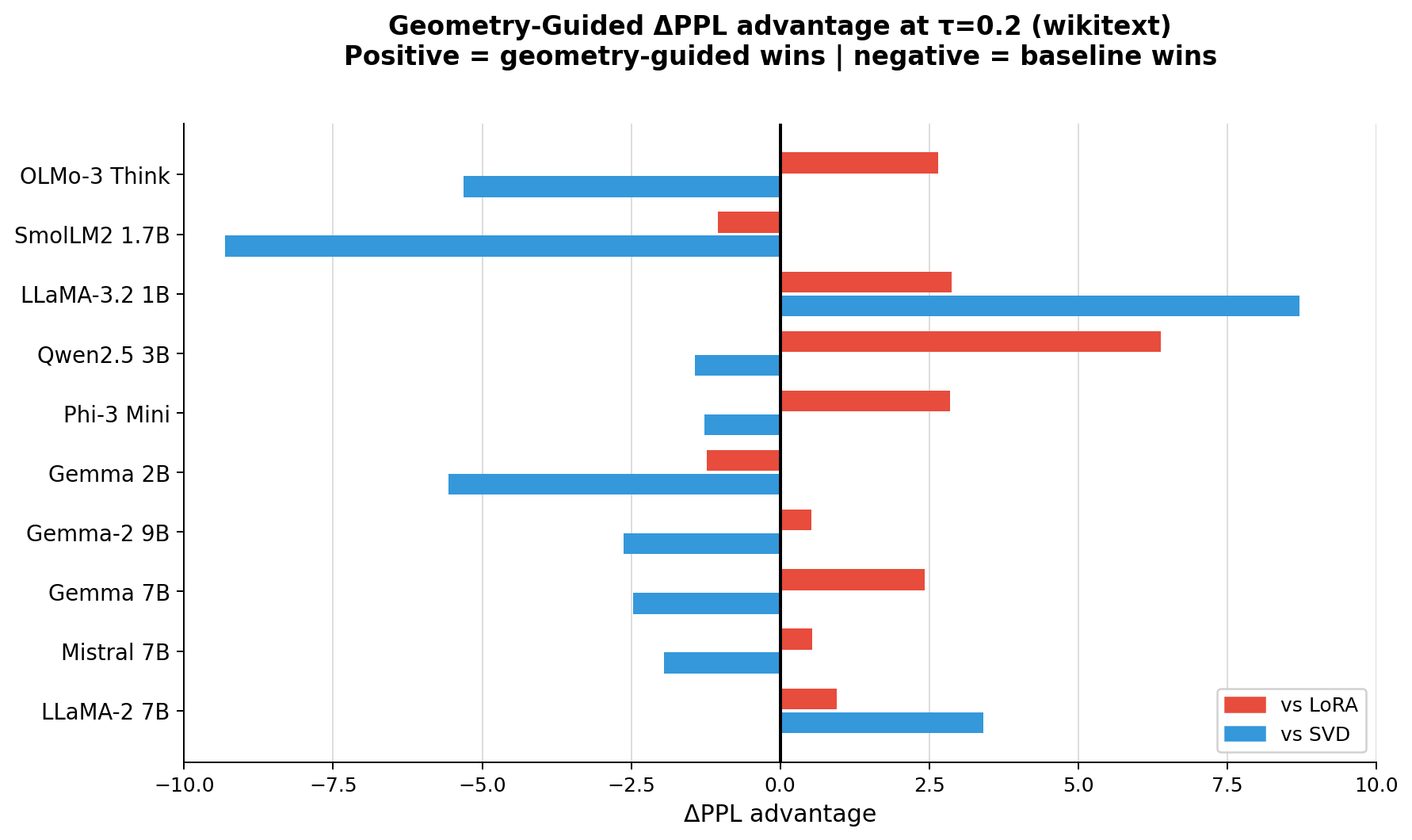}
    \label{fig:exp2_tau02}
  \end{subfigure}

  \caption{\textbf{Experiment 2: geometry-guided \(\Delta\mathrm{PPL}\) advantage on \textsc{WikiText} for different spectral thresholds \(\tau\).} Each panel reports $\Delta\mathrm{PPL}_{\mathrm{adv}} = \Delta\mathrm{PPL}_{\mathrm{uniform}} -
\Delta\mathrm{PPL}_{\mathrm{geom}}$; larger values indicate that geometry-guided rank allocation degrades perplexity less than the best matched uniform baseline.}
  \label{fig:exp2_advantage_wikitext_all_tau}
\end{figure}

\subsection{Experiment 3: geometry-guided pruning}
\label{sec:exp3}

Experiment~3 studies whether the same geometric signals can guide structured
pruning of layer updates. Starting from a compressed model, we rank tokenwise
updates within each layer according to four policies:
(i) the geometric saliency score \(\kappa_{\ell,s}\),
(ii) an attention-score baseline,
(iii) a gradient-activation saliency baseline $
\|\nabla_{X_{\ell,s}}\mathcal L \odot X_{\ell,s}\|_1,
$
and
(iv) a random baseline.
For a drop fraction \(\rho\in\{0.05,0.10,0.20,0.30\}\), we prune the least
important fraction of token-indexed updates under each policy and measure the
resulting validation perplexity increase
\begin{equation}
\Delta\mathrm{PPL}
=
\mathrm{PPL}_{\mathrm{pruned}}
-
\mathrm{PPL}_{\mathrm{unpruned}},
\end{equation}
where lower is better.
We run this experiment on the same family of decoder-only transformers as in
Experiment~2 across \textsc{WikiText}, \textsc{OpenWebText}, and
\textsc{FineWeb}. Figure\,\ref{fig:exp3_pruning_wikitext_all_drop} shows that both structured criteria, \(\kappa_{\ell,s}\) and attention score, substantially outperform random pruning across drop fractions. The \(\kappa\)-based policy is competitive with attention-based pruning and is slightly better in several model/drop-fraction settings. This suggests that the geometric score captures a loss-aware notion of token importance while attention remains a strong baseline for the pruning intervention studied here.

\subsection{Experiment 4: predictive-geometry distillation of low-rank students}
\label{sec:pg-distillation}

The preceding experiments test predictive geometry as a static signal: it
identifies sensitive perturbation directions, allocates low-rank capacity across
layers, and ranks tokenwise updates for pruning. We now ask a stronger question:
can the same geometry improve \emph{training} after compression? This turns the
pullback-Fisher object from a diagnostic into a recovery objective for compressed
students.

Given a teacher $T$, we construct a compressed student $S$ by replacing selected
attention and MLP projections with rank-$r$ low-rank factors,
\begin{equation}
    W \approx AB,
    \qquad
    A \in \mathbb{R}^{d_{\mathrm{out}}\times r},
    \qquad
    B \in \mathbb{R}^{r\times d_{\mathrm{in}}}.
\end{equation}
The student is trained with an output-level KD objective~\cite{hinton2015distilling} and, optionally, a
predictive-geometry regularizer on hidden states. The geometry term penalizes
student errors according to their teacher-predictive effect rather than their
Euclidean size: errors in output-sensitive directions are penalized more heavily,
while errors in approximately prediction-null directions are allowed to be larger.

We separate the role of compression from the role of distillation by comparing
two initializers, SVD and activation-aware SVD, and three KD objectives: forward
KL~\cite{hinton2015distilling}, reverse KL~\cite{gu2024minillm}, and skew KL~\cite{ko2024distillm}. For each KD objective, we evaluate both the base
variant and its geometry-regularized counterpart. The main paired comparisons are
therefore RKL versus RKL+Geo and SkewKL versus SkewKL+Geo at matched model,
rank, seed, and initializer. Full implementation details are given in
Appendix~\ref{app:pg-distillation}.

\begin{table}[!h]
\centering
\caption{Geometry advantage over matched KD baselines. $\Delta_{\mathrm{Geo}}=\mathrm{PPL}(\mathrm{KD})-\mathrm{PPL}(\mathrm{KD+Geo})$; higher is better.}
\label{tab:kd-geom-adv-summary}
\begin{tabular}{lrrrrrr}
\toprule
KD baseline & Pairs & Wins & Mean $\Delta_{\mathrm{Geo}}\uparrow$ & Median $\Delta_{\mathrm{Geo}}\uparrow$ & SVD mean & ASVD mean \\
\midrule
FKL & 14 & 8 & -38.4 & -12.6 & -35.0 & -41.7 \\
RKL & 14 & 9 & 21.3 & 12.3 & 44.5 & -1.9 \\
SkewKL & 14 & 10 & 59.3 & 8.6 & 66.7 & 51.9 \\
\bottomrule
\end{tabular}
\end{table}

The geometry term is most useful when added to stronger autoregressive KD objectives.
While adding geometry to forward KL does not improve average PPL, it improves reverse KL in 9/14 matched settings and skew KL in 10/14 matched settings.
The largest average gain is obtained for skew KL, suggesting that predictive geometry is complementary to mode-seeking or skewed KD objectives rather than simply replacing output-level distillation.

\section{Conclusion}

We introduced a local Fisher-geometric view of hidden trajectories in decoder-only transformers. Starting from layerwise loss-to-go sets, we showed that, near successful validation trajectories and in the low-loss regime, the terminal prediction loss induces a pullback Fisher operator \(K_\ell\) on intermediate hidden states. This operator separates output-sensitive directions from approximately prediction-null directions, yielding a local observable structure inside the residual stream.

For causal transformers, the same geometry induces a tokenwise curvature score \(\kappa_{\ell,s}\), which measures the contribution of each token position to local output-sensitive curvature and vanishes outside the causal ancestor set of the target. Together, the spectrum of \(K_\ell\) and its tokenwise traces provide loss-aware signals for where computation is concentrated, which directions should be preserved, and which components can be compressed with limited predictive damage.

Empirically, the proposed geometry gives a coherent and useful picture across multiple reduction settings. It predicts perturbation sensitivity, supports matched-budget layerwise rank allocation, yields competitive structured pruning signals, and provides an effective regularizer for low-rank student recovery when combined with stronger autoregressive distillation objectives. These results suggest that predictive geometry is not merely a descriptive diagnostic, but a practical mechanism for identifying and preserving the parts of the residual stream that matter for the model's predictions.

Overall, our findings support the view that transformer computation near successful trajectories is organized around a thin, anisotropic set of prediction-relevant directions induced by the terminal loss. Measuring this
geometry turns the final prediction objective into actionable local information inside the network, opening a path toward geometry-aware compression, distillation, and efficient inference methods that are directly tied to the model's predictive behavior.
\newpage

\bibliographystyle{abbrv}
\bibliography{references}

\newpage
\appendix

\section{Proofs}
\label{ap:theory}
\label{app:proofs}

This appendix gives the proofs of the theoretical results stated in
Section~3. We use the same notation as in the main text. In particular,
\(\Psi_\ell:\mathcal X_t\to\mathbb R^{|\mathcal V|}\) denotes the downstream
logit map from layer \(\ell\) to the target-position logits, and
\[
\mathcal J_\ell(X)
=
\phi(\Psi_\ell(X)),
\qquad
\phi(z)
=
D_{\mathrm{KL}}\!\left(q\,\|\,\operatorname{softmax}(z)\right).
\]
We write \(p(z)=\operatorname{softmax}(z)\). The centering projection
\[
P=I-\frac{1}{|\mathcal V|}\mathbf 1\mathbf 1^\top
\]
removes the logit-shift direction, which is invisible to the softmax.

\subsection{Preliminaries: softmax Fisher geometry}
\label{app:softmax_fisher}

We first record elementary identities for the KL loss on logits.

\begin{lemma}[Softmax Fisher identities]
\label{lem:softmax_fisher}
Let
\[
\phi(z)
=
D_{\mathrm{KL}}\!\left(q\,\|\,\operatorname{softmax}(z)\right),
\qquad
p=\operatorname{softmax}(z).
\]
Then
\[
\nabla_z\phi(z)=p-q,
\qquad
\nabla_z^2\phi(z)=F(p):=\operatorname{diag}(p)-pp^\top.
\]
Moreover,
\[
F(p)\mathbf 1=0,
\]
so the Hessian is invariant to additive shifts of the logits.
\end{lemma}

\begin{proof}
Since
\[
\phi(z)
=
\sum_a q_a\log q_a
-
\sum_a q_a\log p_a(z),
\]
and
\[
\log p_a(z)=z_a-\log\sum_b e^{z_b},
\]
we have
\[
\phi(z)
=
\mathrm{const}
-
\sum_a q_a z_a
+
\log\sum_b e^{z_b}.
\]
Therefore
\[
\nabla_z\phi(z)
=
-q+p.
\]
Differentiating \(p=\operatorname{softmax}(z)\) gives
\[
\frac{\partial p_a}{\partial z_b}
=
p_a(\delta_{ab}-p_b),
\]
and hence
\[
\nabla_z^2\phi(z)
=
\operatorname{diag}(p)-pp^\top.
\]
Finally,
\[
F(p)\mathbf 1
=
\operatorname{diag}(p)\mathbf 1
-
pp^\top\mathbf 1
=
p-p=0.
\]
\end{proof}

\subsection{Proof of Theorem~\ref{th1}}
\label{app:proof_th1}

\begin{proof}
We prove the three claims in order.

\paragraph{1. Local Hessian decomposition.}
Let
\[
\bar\Psi_\ell=P\Psi_\ell,
\qquad
\bar J_\ell=D\bar\Psi_\ell(X_\ell^\ast).
\]
Because the softmax is invariant to additive constants,
\[
\operatorname{softmax}(\Psi_\ell(X))
=
\operatorname{softmax}(\bar\Psi_\ell(X))
\]
up to the harmless choice of centered logit representative. Thus we may
differentiate through \(\bar\Psi_\ell\).

By the second-order chain rule,
\[
\nabla^2\mathcal J_\ell(X)
=
D\bar\Psi_\ell(X)^\top
\nabla_z^2\phi(\bar\Psi_\ell(X))
D\bar\Psi_\ell(X)
+
\sum_{a=1}^{|\mathcal V|}
\frac{\partial\phi}{\partial z_a}(\bar\Psi_\ell(X))
\nabla^2\bar\Psi_{\ell,a}(X).
\]
Evaluating at \(X=X_\ell^\ast\), using Lemma~\ref{lem:softmax_fisher}, and
writing
\[
p^\ast=\operatorname{softmax}(\Psi_\ell(X_\ell^\ast)),
\qquad
F^\ast=\operatorname{diag}(p^\ast)-p^\ast p^{\ast\top},
\]
we obtain
\[
\nabla^2\mathcal J_\ell(X_\ell^\ast)
=
\bar J_\ell^\top F^\ast \bar J_\ell
+
R_\ell,
\]
where
\[
R_\ell
=
\sum_{a=1}^{|\mathcal V|}
(p_a^\ast-q_a)
\nabla^2\bar\Psi_{\ell,a}(X_\ell^\ast).
\]
By Cauchy-Schwarz,
\[
\|R_\ell\|_{\mathrm{op}}
\le
\left(
\sum_{a=1}^{|\mathcal V|}
(p_a^\ast-q_a)^2
\right)^{1/2}
\left(
\sum_{a=1}^{|\mathcal V|}
\|\nabla^2\bar\Psi_{\ell,a}(X_\ell^\ast)\|_{\mathrm{op}}^2
\right)^{1/2}.
\]
The second factor is bounded by \(M_\ell\) by assumption, hence
\[
\|R_\ell\|_{\mathrm{op}}
\le
M_\ell\|p^\ast-q\|_2.
\]
Since \(\|u\|_2\le \|u\|_1\), and Pinsker's inequality gives
\[
\|p^\ast-q\|_1
\le
\sqrt{2D_{\mathrm{KL}}(q\|p^\ast)}
=
\sqrt{2\varepsilon_\ast},
\]
we get
\[
\|R_\ell\|_{\mathrm{op}}
\le
M_\ell\sqrt{2\varepsilon_\ast}.
\]

\paragraph{2. Pullback Fisher geometry.}
Define
\[
K_\ell=\bar J_\ell^\top F^\ast\bar J_\ell.
\]
By Taylor expansion of \(\mathcal J_\ell\) around \(X_\ell^\ast\), for
\(X=X_\ell^\ast+\delta X\),
\[
\mathcal J_\ell(X_\ell^\ast+\delta X)
=
\mathcal J_\ell(X_\ell^\ast)
+
\langle\nabla\mathcal J_\ell(X_\ell^\ast),\delta X\rangle
+
\frac12
\langle \delta X,
\nabla^2\mathcal J_\ell(X_\ell^\ast)\delta X\rangle
+
O(\|\delta X\|^3),
\]
provided the third derivative is bounded in the neighborhood. Substituting the
decomposition above gives
\[
\mathcal J_\ell(X_\ell^\ast+\delta X)
=
\mathcal J_\ell(X_\ell^\ast)
+
\langle\nabla\mathcal J_\ell(X_\ell^\ast),\delta X\rangle
+
\frac12
\langle\delta X,K_\ell\delta X\rangle
+
\frac12
\langle\delta X,R_\ell\delta X\rangle
+
O(\|\delta X\|^3).
\]
Using the bound on \(R_\ell\),
\[
|\langle\delta X,R_\ell\delta X\rangle|
\le
M_\ell\sqrt{2\varepsilon_\ast}\|\delta X\|^2.
\]
Absorbing constants into the \(O(\cdot)\) notation gives
\[
\mathcal J_\ell(X_\ell^\ast+\delta X)
=
\mathcal J_\ell(X_\ell^\ast)
+
\langle\nabla\mathcal J_\ell(X_\ell^\ast),\delta X\rangle
+
\frac12
\langle\delta X,K_\ell\delta X\rangle
+
O(\sqrt{\varepsilon_\ast}\|\delta X\|^2)
+
O(\|\delta X\|^3).
\]
Thus, when the linear term is small, the local low-loss set is approximated by
the ellipsoid induced by the quadratic form
\[
g_\ell(\delta X,\delta X)
=
\langle\delta X,K_\ell\delta X\rangle.
\]
If \(K_\ell u_i=\lambda_i u_i\), then the quadratic approximation has radius
proportional to \(\lambda_i^{-1/2}\) in direction \(u_i\), so high-curvature
directions are thin and low-curvature directions are thick.

\paragraph{3. Prediction-null directions.}
If
\[
\delta X\in\ker \bar J_\ell,
\]
then
\[
D\bar\Psi_\ell(X_\ell^\ast)[\delta X]=0.
\]
Thus such directions do not change the centered target logits to first order.
They are therefore first-order prediction-null directions.

If \(\bar J_\ell\) has locally constant rank \(r\), the constant-rank theorem
implies that \(\ker\bar J_\ell\) forms a local null foliation of dimension
\(\dim(\mathcal X_t)-r\). The first-order observable variation is represented
on the quotient
\[
\mathcal O_\ell=\mathcal X_t/\ker\bar J_\ell,
\]
or equivalently on any chosen \(r\)-dimensional complement of the null space.
\end{proof}

Also based on the Theorem\,\ref{th1} we have:

\begin{corollary}[Teacher-KL local geometry]
\label{cor:teacher_kl}
Let \(X_\ell^\star\) be a reference hidden state and set
\[
q=p^\star:=\mathrm{softmax}(\Psi_\ell(X_\ell^\star)).
\]
Then \(X_\ell^\star\) is a stationary point of the teacher-KL loss-to-go
\(J_\ell(X)=D_{\mathrm{KL}}(p^\star\|
\mathrm{softmax}(\Psi_\ell(X)))\). In particular,
\[
\nabla J_\ell(X_\ell^\star)=0.
\]
Consequently, under the assumptions of Theorem~\ref{th1},
\[
J_\ell(X_\ell^\star+\delta X)-J_\ell(X_\ell^\star)
=
\frac{1}{2}\langle \delta X,K_\ell\delta X\rangle
+
O(\sqrt{\epsilon^\star}\|\delta X\|^2)
+
O(\|\delta X\|^3).
\]
\end{corollary}

For the usual supervised cross-entropy objective, \(q=y\) is the one-hot
next-token target. Then \(p^\star-q\) is generally nonzero, even when the model
predicts the correct token with high probability. The expansion in
Theorem~\ref{th1} therefore contains a potentially important linear term $
\langle \nabla J_\ell(X_\ell^\star),\delta X\rangle. $
In this setting \(K_\ell\) should be interpreted as the local output-curvature
operator rather than as a complete approximation to the supervised loss change.
For this reason, our direct Fisher-validation experiments use the teacher-KL choice to isolate the curvature term. A systematic evaluation of the full linear-plus-quadratic approximation for supervised cross-entropy is left to future work.

\subsection{Proof of Proposition~\ref{pr1}}
\label{app:proof_prop1}

\begin{proof}
Let \(P_s:\mathcal X_t\to\mathbb R^d\) be the projection onto the \(s\)-th
token fiber. Its adjoint \(P_s^\top:\mathbb R^d\to\mathcal X_t\) inserts a
token-level perturbation into position \(s\). By definition,
\[
H_{\ell,s}
=
D_{X_{\ell,s}}\bar\Psi_\ell(X_\ell^\ast)
=
\bar J_\ell P_s^\top.
\]

\paragraph{1. Fisher-Jacobian identity.}
The tokenwise score is
\[
\kappa_{\ell,s}
=
\operatorname{tr}(P_sK_\ell P_s^\top).
\]
Using \(K_\ell=\bar J_\ell^\top F^\ast\bar J_\ell\), we get
\[
P_sK_\ell P_s^\top
=
P_s\bar J_\ell^\top F^\ast\bar J_\ell P_s^\top
=
H_{\ell,s}^\top F^\ast H_{\ell,s}.
\]
Therefore
\[
\kappa_{\ell,s}
=
\operatorname{tr}(H_{\ell,s}^\top F^\ast H_{\ell,s}).
\]
Since \(F^\ast\) is positive semidefinite,
\[
\operatorname{tr}(H_{\ell,s}^\top F^\ast H_{\ell,s})
=
\|F^{\ast 1/2}H_{\ell,s}\|_F^2.
\]

\paragraph{2. Causal support.}
In a decoder-only transformer with a standard causal mask, perturbations can
propagate only forward along causal edges from earlier token positions to later
or equal target positions. If token \(s\) is not in the downstream ancestor set
of the target token \(t\) from layer \(\ell\) to \(L\), then no directed
computational path connects \(X_{\ell,s}\) to the target-position logits.
Consequently,
\[
D_{X_{\ell,s}}\bar\Psi_\ell(X_\ell^\ast)=0,
\]
that is,
\[
H_{\ell,s}=0.
\]
The identity above then gives
\[
\kappa_{\ell,s}=0.
\]

\paragraph{3. Pathwise coupling bound.}
Let \(G_m^{(r\leftarrow u)}\) denote the token-block Jacobian of block \(m\).
By repeated application of the chain rule, the downstream Jacobian from
\((\ell,s)\) to the target token at layer \(L\) can be written as a sum over
directed causal paths:
\[
D_{X_{\ell,s}}X_{L,t}
=
\sum_{\pi\in\Pi_{\ell:s\to t}}
\prod_{(m,u\to r)\in\pi}
G_m^{(r\leftarrow u)} ,
\]
where the product is ordered along the path. Therefore,
\[
\|D_{X_{\ell,s}}X_{L,t}\|_{\mathrm{op}}
\le
\sum_{\pi\in\Pi_{\ell:s\to t}}
\prod_{(m,u\to r)\in\pi}
\|G_m^{(r\leftarrow u)}\|_{\mathrm{op}}.
\]
Since
\[
H_{\ell,s}
=
D Z_t(X_L^\ast)\,
D_{X_{\ell,s}}X_{L,t},
\]
we have
\[
\|H_{\ell,s}\|_{\mathrm{op}}
\le
\|D Z_t(X_L^\ast)\|_{\mathrm{op}}
\sum_{\pi\in\Pi_{\ell:s\to t}}
\prod_{(m,u\to r)\in\pi}
\|G_m^{(r\leftarrow u)}\|_{\mathrm{op}}.
\]
Finally,
\[
\kappa_{\ell,s}
=
\operatorname{tr}(H_{\ell,s}^\top F^\ast H_{\ell,s})
\le
\|F^\ast\|_{\mathrm{op}}\|H_{\ell,s}\|_F^2.
\]
Using the same bound with the Frobenius norm, or absorbing the token hidden
dimension into the constant if operator norms are used, yields
\[
\kappa_{\ell,s}
\le
\|F^\ast\|_{\mathrm{op}}
\|D Z_t(X_L^\ast)\|_{\mathrm{op}}^2
\left(
\sum_{\pi\in\Pi_{\ell:s\to t}}
\prod_{(m,u\to r)\in\pi}
\|G_m^{(r\leftarrow u)}\|_{\mathrm{op}}
\right)^2.
\]
\end{proof}

\subsection{Proof of Proposition~\ref{pr2}}
\begin{proof}
Let
\[
K_\ell=\sum_i\lambda_i u_i u_i^\top,
\qquad
\lambda_1\ge\lambda_2\ge\cdots\ge 0,
\]
be an eigendecomposition of \(K_\ell\). Since
\(\Pi_\ell^{(k)}\) projects onto the span of \(u_1,\dots,u_k\), we can write
\[
\delta X=\sum_i a_i u_i,
\qquad
\delta X_{\mathrm{tail}}
=
(I-\Pi_\ell^{(k)})\delta X
=
\sum_{i>k}a_i u_i.
\]
Hence
\[
\left\langle
\delta X_{\mathrm{tail}},
K_\ell
\delta X_{\mathrm{tail}}
\right\rangle
=
\sum_{i>k}\lambda_i a_i^2
\le
\lambda_{k+1}
\sum_{i>k}a_i^2
=
\lambda_{k+1}
\|\delta X_{\mathrm{tail}}\|^2.
\]
This proves \eqref{eq:geometry_reduction_tail_bound}.

For the reduced state
\[
\widetilde X=X_\ell^\ast+\Pi_\ell^{(k)}\delta X,
\]
Theorem~\ref{th1} gives
\begin{align*}
\mathcal J_\ell(\widetilde X)
={}&
\mathcal J_\ell(X_\ell^\ast)
+
\left\langle
\nabla\mathcal J_\ell(X_\ell^\ast),
\Pi_\ell^{(k)}\delta X
\right\rangle
+
\frac12
\left\langle
\Pi_\ell^{(k)}\delta X,
K_\ell
\Pi_\ell^{(k)}\delta X
\right\rangle
\\
&+
O\!\left(\sqrt{\varepsilon_\ast}
\|\Pi_\ell^{(k)}\delta X\|^2\right)
+
O(\|\Pi_\ell^{(k)}\delta X\|^3).
\end{align*}
Since \(\Pi_\ell^{(k)}\) is an orthogonal projector,
\[
\|\Pi_\ell^{(k)}\delta X\|\le \|\delta X\|,
\]
so the remainder is bounded by
\[
O\!\left(\sqrt{\varepsilon_\ast}\|\delta X\|^2\right)
+
O(\|\delta X\|^3).
\]

Finally, the rank rule follows by using the spectral tail
\[
\sum_{i>k}\lambda_i(K_\ell)
\]
as a layerwise proxy for the total discarded output-sensitive curvature. If
the tail is at most a fraction \(\tau\) of the total curvature, then the
discarded subspace contains only a controlled fraction of the layer's local
predictive curvature.
\end{proof}

\newpage
\section{Experimental Protocol}
\label{app:experimental_protocol}

Appendix B collects experimental details shared across the perturbation, rank-allocation, pruning, and distillation experiments. The goal of these settings is to keep the geometry measurements local and comparable across models: the transformer weights are fixed, calibration and evaluation examples are separated, and curvature quantities are estimated without materializing the full pullback-Fisher matrix.

\paragraph{Models.}
We evaluate open decoder-only language models ranging from approximately 1B to
9B parameters. All models are
evaluated without fine-tuning. We disable key-value caching during curvature
estimation so that hidden-state perturbations are propagated through the full
downstream computation.

\begin{table}[t]
\centering
\caption{Model checkpoints used in the experiments.}
\label{tab:model_checkpoints}
\begin{tabular}{lllr}
\toprule
Model name in paper & Hugging Face checkpoint & Model reference & Params \\
\midrule
SmolLM2-1.7B & \texttt{HuggingFaceTB/SmolLM2-1.7B} &\cite{allal2025smollm2}& 1.7B \\
LLaMA-3.2-1B & \texttt{meta-llama/Llama-3.2-1B} &\cite{grattafiori2024llama3}& 1B \\
Qwen2.5-3B & \texttt{Qwen/Qwen2.5-3B} &\cite{yang2025qwen25}& 3B \\
Phi-3 Mini & \texttt{microsoft/Phi-3-mini-4k-instruct} &\cite{abdin2024phi3}& 3.8B \\
Mistral-7B & \texttt{mistralai/Mistral-7B-v0.1}&\cite{jiang2023mistral} & 7B \\
LLaMA-2-7B & \texttt{meta-llama/Llama-2-7b-hf}&\cite{touvron2023llama2} & 7B \\
Gemma-2B & \texttt{google/gemma-2-2b}& \cite{gemma2024gemma}& 2B \\
Gemma-7B & \texttt{google/gemma-7b}& \cite{gemma2024gemma}& 7B \\
OLMo-3-7B & \texttt{allenai/OLMo-3-7B}& \cite{olmo2025olmo3}& 7B \\

Gemma-2-9B & \texttt{google/gemma-2-9b}& \cite{gemma2024gemma}& 9B \\
\bottomrule
\end{tabular}
\end{table}

\paragraph{Datasets.}
We use held-out text from WikiText, OpenWebText, and FineWeb. For each dataset,
we sample fixed-length token sequences and choose target positions among
non-padding tokens. Calibration examples are used only to estimate geometric
quantities or baseline saliency scores; evaluation examples are disjoint unless
otherwise stated.

\paragraph{Target positions.}
For each sequence, the target position is sampled from the second half of the
non-padding context. This avoids degenerate early-token targets while preserving
causal ancestry over a substantial prefix. All token-level scores are computed
only over causal ancestors of the target position.

\paragraph{Precision and attention implementation.}
Forward passes are run in the inference precision of the model, typically
bfloat16. Curvature products, Rayleigh quotients, trace estimates, and bootstrap
statistics are accumulated in float32. For experiments involving
Jacobian-vector or vector-Jacobian products, we use eager/math attention rather
than fused flash attention, because fused attention kernels are not always
compatible with higher-order automatic differentiation.

\paragraph{Matrix-free curvature products.}
We never materialize the pullback Fisher matrix \(K_\ell\). Products
\(K_\ell v\) are computed by the sequence JVP-Fisher-VJP:
\[
v \mapsto J_\ell v \mapsto F^\star J_\ell v
\mapsto J_\ell^\top F^\star J_\ell v.
\]
The middle multiplication by \(F^\star\) is implemented as
\[
F^\star u
=
p^\star \odot u
-
p^\star \langle p^\star,u\rangle .
\]

\newpage
\section{Numerical Estimation of Predictive Geometry}
\label{app:numerics}

This appendix describes how we estimate the geometric quantities used in
Section~4 without explicitly forming the pullback Fisher matrix \(K_\ell\).
The key observation is that all required quantities can be computed using
matrix-free products with \(K_\ell\), implemented by Jacobian-vector products
(JVPs) and vector-Jacobian products (VJPs). This follows the standard automatic
differentiation approach to matrix-free curvature products, related to
Pearlmutter's Hessian-vector product method~\cite{pearlmutter1994fast}.

\subsection{Matrix-free multiplication by \(K_\ell\)}
\label{app:matrix_free_K}

Recall that
\[
K_\ell
=
\bar J_\ell^\top F^\ast \bar J_\ell,
\qquad
\bar J_\ell=D\bar\Psi_\ell(X_\ell^\ast).
\]
For any vector \(v\in\mathcal X_t\),
\[
K_\ell v
=
\bar J_\ell^\top
F^\ast
(\bar J_\ell v).
\]
Thus a product \(K_\ell v\) can be computed in three steps:
\[
u=\bar J_\ell v,
\qquad
w=F^\ast u,
\qquad
K_\ell v=\bar J_\ell^\top w.
\]
The middle multiplication is cheap because
\[
F^\ast u
=
p^\ast\odot u
-
p^\ast\langle p^\ast,u\rangle.
\]
Therefore, no \(|\mathcal V|\times|\mathcal V|\) Fisher matrix needs to be
materialized.

\begin{algorithm}[h]
\caption{Matrix-free multiplication by \(K_\ell\)}
\label{alg:matrix_free_K}
\begin{algorithmic}[1]
\Require Hidden state \(X_\ell^\ast\), vector \(v\in\mathcal X_t\)
\State \(u \leftarrow D\bar\Psi_\ell(X_\ell^\ast)[v]\) \Comment{JVP}
\State \(w \leftarrow p^\ast\odot u - p^\ast\langle p^\ast,u\rangle\)
\State \(z \leftarrow D\bar\Psi_\ell(X_\ell^\ast)^\top[w]\) \Comment{VJP}
\State \Return \(z\)
\end{algorithmic}
\end{algorithm}
\subsection{PyTorch implementation with JVP and VJP}
\label{app:pytorch_jvp_vjp}

In PyTorch, the downstream logit map
\[
\bar\Psi_\ell:\mathcal X_t\to\mathbb R^{|\mathcal V|}
\]
can be represented as a function that takes a layer-\(\ell\) hidden state and
runs only the remaining transformer blocks, final normalization, and language
model head. Matrix-free products with \(K_\ell\) can then be implemented using
\texttt{torch.func.jvp} and \texttt{torch.func.vjp}.

A schematic implementation is:

\begin{verbatim}
import torch
from torch.func import jvp, vjp

def centered_logits_from_layer(x_l):
    # Runs blocks ell,...,L-1, final norm, and lm_head.
    # Returns target-position logits with mean removed.
    logits = downstream_model_from_layer(x_l)      # shape: [vocab]
    return logits - logits.mean(dim=-1, keepdim=True)

def fisher_mv(p, u):
    # F u = diag(p)u - p(p^T u)
    return p * u - p * torch.sum(p * u, dim=-1, keepdim=True)

def K_mv(x_l_star, v):
    # JVP: u = J v
    logits_star, u = jvp(centered_logits_from_layer,
                         (x_l_star,),
                         (v,))
    p = torch.softmax(logits_star, dim=-1)

    # Fisher multiplication
    w = fisher_mv(p, u)

    # VJP: z = J^T w
    _, vjp_fn = vjp(centered_logits_from_layer, x_l_star)
    z, = vjp_fn(w)
    return z
\end{verbatim}

In practice, one should avoid storing gradients for model parameters when only
hidden-state derivatives are needed. The model weights are kept fixed, and
\texttt{x\_l\_star} is treated as the differentiable input. Mixed precision can
be used for the forward pass, but curvature estimation is often more stable in
\texttt{float32} or with accumulation in \texttt{float32}.

For models with key-value caching, the downstream function should be defined
carefully so that the perturbation at layer \(\ell\) is applied to the hidden
state whose effect is being studied. Caches that would bypass the perturbed
hidden state must be disabled or recomputed.

\subsection{Estimating traces and effective ranks with Hutchinson probes}
\label{app:hutchinson}

For a symmetric matrix \(A\), Hutchinson's estimator gives
\[
\operatorname{tr}(A)
=
\mathbb E_{\xi}[\xi^\top A\xi],
\]
where \(\xi\) has independent Rademacher entries, i.e.
\(\xi_i\in\{-1,+1\}\) with equal probability. With \(m\) probes,
\[
\widehat{\operatorname{tr}}(A)
=
\frac1m\sum_{j=1}^m \xi_j^\top A\xi_j.
\]
This estimator is unbiased for \(\operatorname{tr}(A)\).

For \(A=K_\ell\), each term requires one matrix-free product \(K_\ell\xi_j\):
\[
\xi_j^\top K_\ell\xi_j
=
\langle \xi_j, K_\ell\xi_j\rangle.
\]
Thus
\[
\widehat{\operatorname{tr}}(K_\ell)
=
\frac1m\sum_{j=1}^m
\langle \xi_j,K_\ell\xi_j\rangle.
\]

The participation-ratio effective rank used in the main text is
\[
r_{\mathrm{eff}}(K_\ell)
=
\frac{\operatorname{tr}(K_\ell)^2}{\operatorname{tr}(K_\ell^2)}.
\]
The denominator can also be estimated by Hutchinson:
\[
\operatorname{tr}(K_\ell^2)
=
\mathbb E_{\xi}
[
\|K_\ell\xi\|^2
],
\]
so
\[
\widehat{\operatorname{tr}}(K_\ell^2)
=
\frac1m
\sum_{j=1}^m
\|K_\ell\xi_j\|^2.
\]
This requires one \(K_\ell\)-product per probe if \(K_\ell\xi_j\) is stored
for both the trace and squared-trace estimates.

\begin{algorithm}[h]
\caption{Hutchinson estimates of \(\operatorname{tr}(K_\ell)\) and
\(\operatorname{tr}(K_\ell^2)\)}
\label{alg:hutchinson}
\begin{algorithmic}[1]
\Require Matrix-free routine \(v\mapsto K_\ell v\), number of probes \(m\)
\State \(s_1\leftarrow 0,\quad s_2\leftarrow 0\)
\For{\(j=1,\dots,m\)}
    \State Sample Rademacher probe \(\xi_j\)
    \State \(y_j\leftarrow K_\ell\xi_j\)
    \State \(s_1\leftarrow s_1+\langle \xi_j,y_j\rangle\)
    \State \(s_2\leftarrow s_2+\langle y_j,y_j\rangle\)
\EndFor
\State \(\widehat{\operatorname{tr}}(K_\ell)\leftarrow s_1/m\)
\State \(\widehat{\operatorname{tr}}(K_\ell^2)\leftarrow s_2/m\)
\State \(\widehat r_{\mathrm{eff}}\leftarrow
\widehat{\operatorname{tr}}(K_\ell)^2/
\widehat{\operatorname{tr}}(K_\ell^2)\)
\State \Return \(\widehat{\operatorname{tr}}(K_\ell)\),
\(\widehat{\operatorname{tr}}(K_\ell^2)\),
\(\widehat r_{\mathrm{eff}}\)
\end{algorithmic}
\end{algorithm}

\subsection{Practical implementation details}
\label{app:implementation_details}

\paragraph{Choice of target distribution \(q\).}
For local geometry around a trained model trajectory, we use either the
one-hot next-token target or the model's own unperturbed prediction as a
teacher distribution. The teacher-KL choice
\[
q=p^\ast
\]
removes the first-order term at the reference point and directly measures
local sensitivity of the output distribution to hidden-state perturbations.
The one-hot choice corresponds to the usual supervised cross-entropy geometry
but may have a nonzero first-order term.

\paragraph{Batching.}
All estimates are averaged over calibration examples and target positions.
For a batch \(\mathcal B\), we estimate
\[
\bar K_\ell
=
\frac1{|\mathcal B|}
\sum_{(x,t)\in\mathcal B}
K_\ell(x,t).
\]
Matrix-free products with \(\bar K_\ell\) are implemented by averaging
per-example products.

\paragraph{Numerical precision.}
Curvature estimates can be sensitive to numerical precision. We run the model
forward pass in the precision used for inference, but accumulate Hutchinson
statistics, Lanczos inner products, and Rayleigh quotients in \texttt{float32}.
For very small models or diagnostic experiments, \texttt{float64} can be used
to validate the estimates.

\paragraph{Stopping criteria.}
For Lanczos, we use a fixed number of iterations \(q\) or stop when the change
in the estimated leading eigenspace is below a tolerance. For Hutchinson trace
estimation, we report the number of probes and use bootstrap confidence
intervals over probes and calibration examples when reporting trace-derived
quantities.

\paragraph{Memory.}
The method does not materialize \(K_\ell\). Its memory overhead is dominated by
the downstream forward graph needed for JVP/VJP and by the storage required for
Lanczos basis vectors. For large models, one may checkpoint downstream blocks or
process layers independently.

\section{Additional Details for Local-Geometry and Token-Pruning Experiments}
\label{app:exp-details-local-pruning}

This section provides additional details for the experiments used to validate the predictive-geometry interpretation of the pullback Fisher operator. The goal of these experiments is not to show that a single Fisher-based heuristic dominates all possible alternatives, but rather to test two specific consequences of the theory: (i) the terminal prediction loss induces a measurable local geometry on intermediate hidden states, and (ii) the token-fiber trace of this geometry provides a useful loss-aware token saliency signal.

For a fixed trained decoder-only transformer, let \(X_\ell^\star\) denote the residual-stream state at layer \(\ell\) along the unperturbed forward pass. Let the downstream computation from layer \(\ell\) to the target-position logits be denoted by \(\Psi_\ell\). Around \(X_\ell^\star\), the terminal predictive distribution induces the pullback Fisher operator
\[
K_\ell
=
J_\ell^\top F^\star J_\ell,
\]
where \(J_\ell\) is the Jacobian of the downstream map from \(X_\ell\) to centered target logits, and
\[
F^\star
=
\operatorname{diag}(p^\star) - p^\star p^{\star\top}
\]
is the softmax Fisher matrix at the reference output distribution \(p^\star\). Intuitively, \(K_\ell\) measures which perturbations of the hidden state are locally visible to the terminal predictive distribution.

\subsection{Local Taylor Validation}
\label{app:exp1-local-taylor}

The first experiment tests whether the local loss change induced by hidden-state perturbations is predicted by the Taylor expansion implied by the pullback Fisher geometry. This directly checks the central local-geometry claim: perturbations in hidden-state directions with large Fisher pullback curvature should produce larger changes in the terminal prediction.

For each model, layer \(\ell\), and target token position, we first run the model normally and record the reference hidden state \(X_\ell^\star\). We then sample random perturbations \(\delta X_\ell\) at several radii and inject \(X_\ell^\star + \delta X_\ell\) at layer \(\ell\), while keeping the downstream transformer fixed. We compare the true change in terminal loss to local predictors computed at the reference trajectory.

For supervised cross-entropy, the local approximation includes both the first-order term and the pullback-Fisher quadratic term:
\[
\Delta \mathrm{CE}
\approx
\left\langle \nabla_{X_\ell}\mathrm{CE}, \delta X_\ell \right\rangle
+
\frac{1}{2}
\delta X_\ell^\top K_\ell \delta X_\ell.
\]
We also evaluate two ablations:
\[
\Delta \mathrm{CE}_{\mathrm{linear}}
=
\left\langle \nabla_{X_\ell}\mathrm{CE}, \delta X_\ell \right\rangle,
\]
and
\[
\Delta \mathrm{CE}_{\mathrm{quad}}
=
\frac{1}{2}
\delta X_\ell^\top K_\ell \delta X_\ell.
\]
This distinction is important because supervised one-hot CE is not generally stationary at the reference prediction, so the linear term need not vanish.

For teacher-KL, we instead compare the true change in the KL divergence from the reference predictive distribution to the Fisher quadratic predictor:
\[
\Delta \mathrm{KL}(p^\star \,\|\, p_{\delta})
\approx
\frac{1}{2}
\delta X_\ell^\top K_\ell \delta X_\ell.
\]
In this setting, the reference point is stationary, so the Fisher quadratic term is the leading local contribution.

For each layer and perturbation radius, we report the coefficient of determination \(R^2\) and correlation between the true loss change and the corresponding local predictor. We aggregate results across models by first averaging within each model and radius, and then averaging across models so that each model receives equal weight. For visualization only, degenerate \(R^2\) values are clipped to \([-1,1]\), since very large negative \(R^2\) values can occur when the true loss changes have near-zero variance. Outlier diagnostics are retained separately.

If the local predictive geometry is meaningful, we expect the teacher-KL quadratic approximation to explain a substantial fraction of the local KL change at perturbation radii where the signal is above numerical noise but still within the local regime. For supervised CE, we expect the linear-plus-quadratic approximation to outperform the quadratic-only approximation. We do not necessarily expect a large gain over the linear-only predictor at very small radii, because supervised CE can be first-order dominated.

Thus the expected qualitative pattern is:
\[
\mathrm{CE}_{\mathrm{linear+quadratic}}
\approx
\mathrm{CE}_{\mathrm{linear}}
\gg
\mathrm{CE}_{\mathrm{quadratic-only}}
\]
for small perturbations, while
\[
\mathrm{TeacherKL}_{\mathrm{quadratic}}
\]
should become predictive once the perturbation-induced KL signal is sufficiently measurable. This validates the interpretation of \(K_\ell\) as a curvature or observability operator, rather than as a complete supervised-CE predictor by itself.

\begin{figure}
    \centering
    \includegraphics[width=\linewidth]{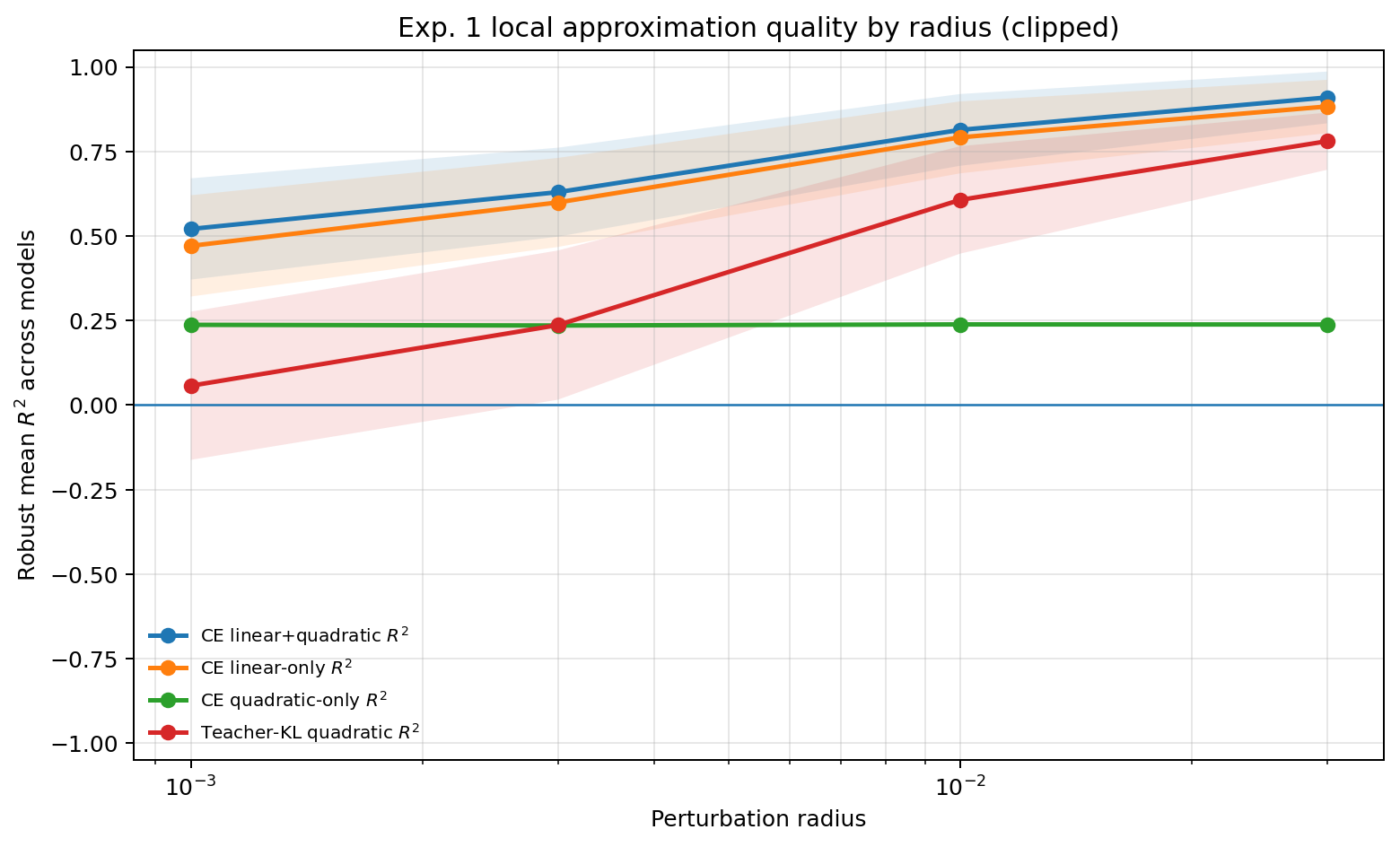}
    \caption{
\textbf{Local hidden-state geometry predicts terminal loss changes.}
We perturb intermediate residual-stream states at different radii and compare the true terminal loss change to local predictors induced by the downstream map.
For supervised CE, the linear-only and linear-plus-quadratic predictors are nearly indistinguishable, indicating that CE is first-order dominated in this perturbation regime.
The quadratic-only CE predictor is substantially weaker, as expected because supervised one-hot CE is generally not stationary at the reference prediction.
In contrast, the teacher-KL quadratic predictor isolates the pullback-Fisher curvature and becomes strongly predictive at moderate perturbation radii.
Curves show model-averaged \(R^2\); for visualization, \(R^2\) values are clipped to \([-1,1]\) before aggregation to avoid degenerate-regression outliers.
}
\label{fig:exp1-r2-by-radius}
    \label{fig:placeholder}
\end{figure}

\paragraph{Interpretation.}
A strong result for this experiment is not that the Fisher quadratic term alone predicts supervised CE. Rather, the intended conclusion is more precise: supervised CE requires the first-order term, while teacher-KL isolates the Fisher curvature. Therefore, the experiment separates two roles of the local geometry. The gradient term explains first-order supervised-loss sensitivity, whereas the pullback Fisher term captures second-order output-distribution sensitivity.

\subsection{Token-Fiber Saliency and Structured Token-Update Pruning}
\label{app:exp3-token-pruning}

The third experiment evaluates whether the tokenwise trace of the pullback Fisher operator provides a useful loss-aware saliency score for token updates. The theory predicts that if a token position has small pullback-Fisher trace, then perturbations localized to that token position are less visible to the terminal predictive distribution. Such token updates should therefore be safer to prune.

Let \(P_s\) denote the projection onto the residual-stream coordinates corresponding to token position \(s\). We define the tokenwise Fisher trace
\[
\kappa_{\ell,s}
=
\operatorname{tr}
\left(
P_s K_\ell P_s^\top
\right).
\]
Equivalently, \(\kappa_{\ell,s}\) measures the squared downstream sensitivity of token \(s\)'s hidden state, weighted by the output Fisher geometry. Large values indicate token positions whose hidden-state perturbations are locally visible to the final prediction, while small values indicate approximately prediction-insensitive token positions.

For each example and selected layer, we compute saliency scores over token positions. We then prune a fixed fraction \(\rho\) of token updates according to each score, selecting the least salient positions for removal. The model is evaluated after the pruning intervention, and we measure the increase in cross-entropy relative to the unpruned forward pass:
\[
\Delta \mathrm{CE}
=
\mathrm{CE}_{\mathrm{pruned}}
-
\mathrm{CE}_{\mathrm{base}}.
\]
Lower \(\Delta \mathrm{CE}\) indicates a better pruning criterion.

In the main analysis, we report \(\Delta \mathrm{CE}\) rather than \(\Delta\mathrm{PPL}\). Since perplexity exponentiates CE, rare catastrophic pruning failures can dominate mean \(\Delta\mathrm{PPL}\) and obscure the typical behavior. The CE difference is also equal to the log-perplexity ratio, making it a more stable primary metric:
\[
\Delta \mathrm{CE}
=
\log \frac{\mathrm{PPL}_{\mathrm{pruned}}}{\mathrm{PPL}_{\mathrm{base}}}.
\]

We compare the Fisher trace score against several alternative token saliency criteria:
\begin{itemize}
    \item \textbf{Random}: randomly chosen token positions.
    \item \textbf{Activation norm}: token hidden-state norm.
    \item \textbf{Attention}: mean attention mass to the target position.
    \item \textbf{\(J^\top J\) trace}: unweighted downstream Jacobian sensitivity.
    \item \textbf{Diagonal Fisher}: a diagonal approximation to the Fisher weighting.
    \item \textbf{Gradient-based scores}: first-order loss heuristics such as gradient-times-activation and top-logit gradient.
\end{itemize}

These baselines separate several possible explanations. Random pruning checks whether any signal is present at all. Activation norm tests whether large hidden activations are sufficient. Attention tests whether standard attention weights already capture the relevant token saliency. The \(J^\top J\) baseline tests whether the Fisher weighting itself adds value beyond unweighted downstream sensitivity. Gradient-based scores provide strong first-order loss-aware baselines for the specific pruning intervention.

For each method and pruning fraction \(\rho\), we report median or mean \(\Delta \mathrm{CE}\) across models. Aggregation is performed by first averaging within each model, method, and pruning fraction, and then averaging across models so that each model has equal weight.

We also report a Fisher-advantage plot:
\[
\Delta \mathrm{CE}_{\mathrm{baseline}}
-
\Delta \mathrm{CE}_{\mathrm{Fisher}}.
\]
Positive values indicate that Fisher pruning causes a smaller CE increase than the baseline. This relative plot is useful for comparing Fisher directly to each alternative, but it should be interpreted together with the absolute \(\Delta \mathrm{CE}\) Pareto curves.

If the tokenwise pullback-Fisher trace captures prediction-relevant information, then Fisher pruning should be substantially better than random pruning and should outperform naive activation-based criteria. We also expect it to be competitive with attention and Jacobian-style baselines. We do not require Fisher to dominate all gradient-based heuristics, because those criteria use direct first-order loss information tailored to the specific pruning intervention.

The expected qualitative ranking is therefore:
\[
\mathrm{Random}
\quad \text{and naive activation criteria}
\quad \text{worse than Fisher},
\]
while Fisher should lie near the stronger loss-aware or geometry-aware baselines:
\[
\mathrm{Fisher}
\approx
J^\top J
\approx
\mathrm{Attention},
\]
with direct gradient-based scores potentially competitive or stronger depending on the intervention.

\begin{figure}
    \centering
    \includegraphics[width=\linewidth]{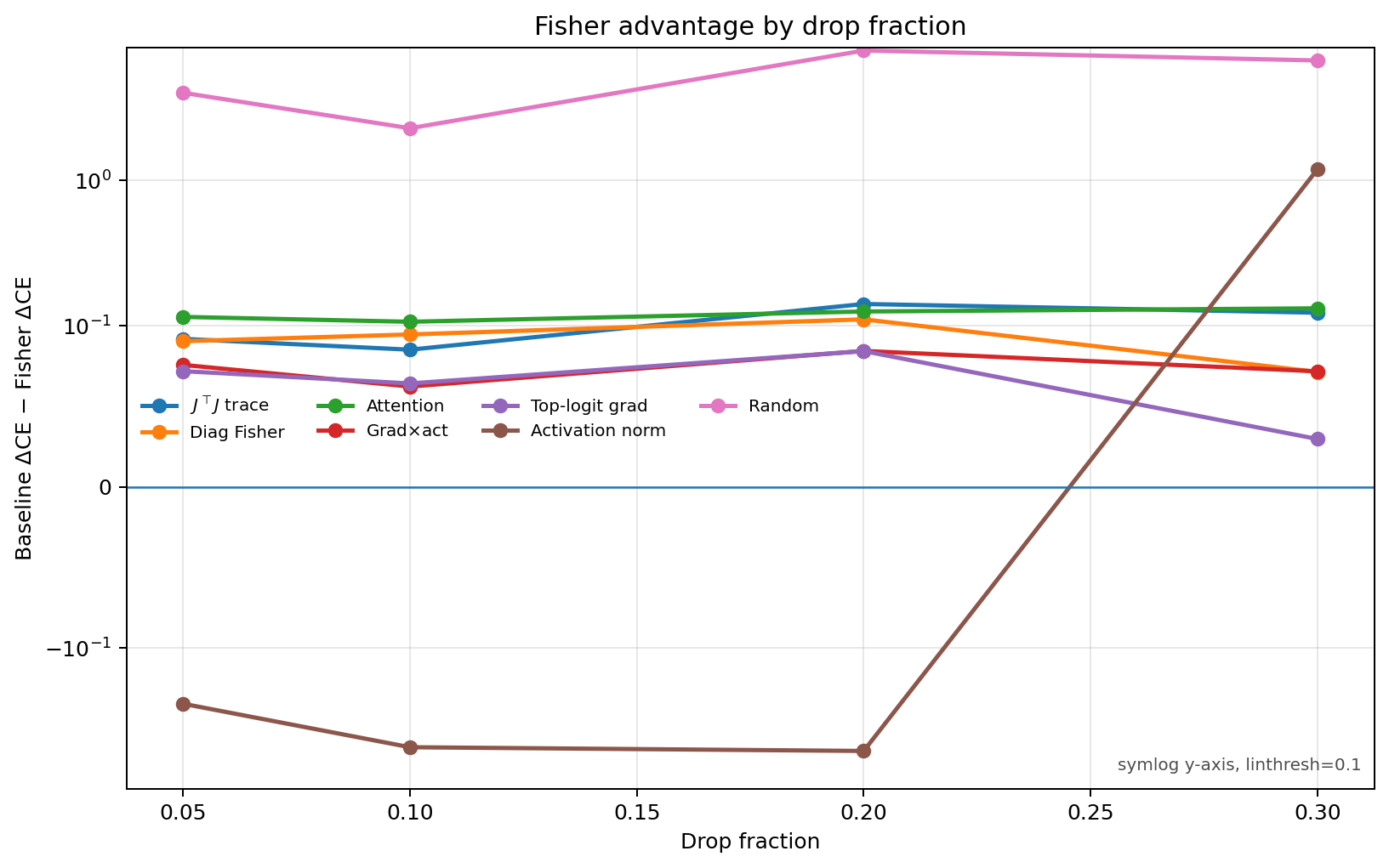}
    \caption{
\textbf{Fisher token saliency is competitive with standard pruning baselines.}
We prune a fraction of token updates according to each baseline score and report the Fisher advantage
\(\Delta \mathrm{CE}_{\mathrm{baseline}} - \Delta \mathrm{CE}_{\mathrm{Fisher}}\).
Positive values indicate that Fisher pruning produces a smaller CE increase than the baseline.
Fisher substantially outperforms random pruning and gives modest, consistent improvements over attention, \(J^\top J\), diagonal-Fisher, and gradient-times-activation criteria in the aggregate.
Activation norm is less stable across pruning fractions, performing better at small drop fractions but worse at the largest drop fraction.
The \(y\)-axis uses a symlog scale with linear threshold \(0.1\) to show both small differences near zero and large improvements over weak baselines.
}
\label{fig:exp3-fisher-advantage}
\end{figure}

\paragraph{Interpretation.}
This experiment should be interpreted as a practical validation that the pullback-Fisher trace contains usable token-level predictive information. A positive result is not necessarily that Fisher is the universally optimal pruning rule. Rather, the key conclusion is that the Fisher trace is a non-random, loss-aware geometric saliency score that identifies token updates whose removal produces smaller degradation than random or naive pruning. This supports the broader predictive-geometry view: the terminal loss induces a structured notion of local observability over token positions in the residual stream.

\newpage
\section{Details for Predictive-Geometry Distillation}
\label{app:pg-distillation}

\subsection{Low-Rank Student Construction}

For each teacher model, we construct a compressed student by replacing selected linear layers with low-rank factors. We compress the attention projections
\begin{equation}
q_{\mathrm{proj}},\quad
k_{\mathrm{proj}},\quad
v_{\mathrm{proj}},\quad
o_{\mathrm{proj}},
\end{equation}
and the MLP projections
\begin{equation}
\mathrm{gate}_{\mathrm{proj}},\quad
\mathrm{up}_{\mathrm{proj}},\quad
\mathrm{down}_{\mathrm{proj}}.
\end{equation}
A full linear map $W$ is replaced by
\begin{equation}
W \approx A B,
\qquad
A \in \mathbb{R}^{d_{\mathrm{out}}\times r},
\qquad
B \in \mathbb{R}^{r\times d_{\mathrm{in}}}.
\end{equation}
Only the low-rank factors are trained during recovery; all other student parameters are frozen.

We evaluate two initializers. The first is ordinary truncated SVD. The second is activation-aware SVD. For activation-aware SVD, we collect diagonal activation second moments
\begin{equation}
s_i^2 = \mathbb{E}[x_i^2],
\end{equation}
where $x$ is the input to the corresponding linear layer. We then compute a rank-$r$ approximation to the scaled matrix
\begin{equation}
W \operatorname{diag}(s)
\end{equation}
and unscale the right factor. This minimizes the activation-weighted reconstruction objective
\begin{equation}
\|(W-\hat W)\operatorname{diag}(s)\|_F^2.
\end{equation}
This controls for the possibility that geometry-aware distillation only helps because ordinary SVD is a weak initializer.

\subsection{KD Objectives}

We evaluate three output-level distillation objectives.

\paragraph{Forward KL.}
\begin{equation}
\mathcal{L}_{\mathrm{FKL}}
=
D_{\mathrm{KL}}(p_T \Vert p_S)
=
\sum_y p_T(y) \log \frac{p_T(y)}{p_S(y)}.
\end{equation}

\paragraph{Reverse KL.}
\begin{equation}
\mathcal{L}_{\mathrm{RKL}}
=
D_{\mathrm{KL}}(p_S \Vert p_T)
=
\sum_y p_S(y) \log \frac{p_S(y)}{p_T(y)}.
\end{equation}
Reverse KL is a stronger autoregressive distillation baseline because it is more mode-seeking and can be better matched to a lower-capacity student.

\paragraph{Skew KL.}
\begin{equation}
\mathcal{L}_{\mathrm{SkewKL}}
=
D_{\mathrm{KL}}
\left(
p_T
\middle\Vert
\alpha p_T + (1-\alpha)p_S
\right).
\end{equation}
We use $\alpha=0.1$ unless otherwise stated. For all KD objectives, logits are divided by temperature $\tau$, and the KD loss is multiplied by $\tau^2$. We use $\tau=2$.

\subsection{Predictive-Geometry Regularization}

For each KD objective, we evaluate a geometry-regularized variant:
\begin{equation}
\mathcal{L}
=
\mathcal{L}_{\mathrm{KD}}
+
\lambda_{\mathrm{CE}}\mathcal{L}_{\mathrm{CE}}
+
\lambda_{\mathrm{geo}}\mathcal{L}_{\mathrm{geo}}.
\end{equation}

The ideal geometry term is the local pullback-Fisher quadratic
\begin{equation}
\mathcal{L}_{\mathrm{geo}}
=
\frac{1}{|\mathcal{L}_{\mathrm{geo}}|}
\sum_{\ell \in \mathcal{L}_{\mathrm{geo}}}
\frac{1}{2}
\delta h_\ell^\top K^T_\ell \delta h_\ell,
\qquad
\delta h_\ell = h^S_\ell - h^T_\ell,
\end{equation}
where
\begin{equation}
K^T_\ell = J_\ell^\top F J_\ell
\end{equation}
is the teacher pullback-Fisher operator induced by the downstream teacher map from layer $\ell$ to the output distribution.

In the current implementation, we use a readout-Fisher proxy for computational robustness across model families:
\begin{equation}
h_\ell
\mapsto
\mathrm{LMHead}(\mathrm{FinalNorm}(h_\ell)).
\end{equation}
Thus,
\begin{equation}
\mathcal{L}_{\mathrm{geo}}
=
\frac{1}{|\mathcal{L}_{\mathrm{geo}}|}
\sum_{\ell \in \mathcal{L}_{\mathrm{geo}}}
\frac{1}{2}
\Delta z_\ell^\top F_T \Delta z_\ell,
\end{equation}
where
\begin{equation}
\Delta z_\ell
=
g_T(h^S_\ell) - g_T(h^T_\ell),
\qquad
g_T(h)=\mathrm{LMHead}_T(\mathrm{FinalNorm}_T(h)).
\end{equation}
This preserves the output-Fisher weighting of hidden-state errors, but is weaker than the full downstream pullback-Fisher operator. We therefore report it as a readout-Fisher predictive-geometry regularizer.

\subsection{Experimental Grid}

We evaluate two initializers:
\begin{equation}
\{\mathrm{SVD}, \mathrm{ASVD}\}.
\end{equation}
For each initializer, we evaluate:
\begin{equation}
\{
\mathrm{FKL},
\mathrm{RKL},
\mathrm{SkewKL},
\mathrm{FKL+Geo},
\mathrm{RKL+Geo},
\mathrm{SkewKL+Geo}
\}.
\end{equation}

The sweep is run on sub-billion-parameter models:
\begin{equation}
\mathrm{SmolLM2\text{-}135M},\quad
\mathrm{SmolLM2\text{-}360M},\quad
\mathrm{Qwen2.5\text{-}0.5B}.
\end{equation}
Unless otherwise stated, we use ranks
\begin{equation}
r \in \{16,32,64\},
\end{equation}
sequence length 256, batch size 1, temperature $\tau=2$, $\lambda_{\mathrm{CE}}=0.1$, and $\lambda_{\mathrm{geo}}=0.05$. The geometry term is evaluated every 4 optimization steps on layers
\begin{equation}
\{4,8,12,16,20\}.
\end{equation}
For ASVD, activation statistics are collected from 16 calibration batches.

\subsection{Metrics}

We report final perplexity,
\begin{equation}
\mathrm{PPL}(S),
\end{equation}
where lower is better; the compression gap to the teacher,
\begin{equation}
\Delta_{\mathrm{teacher}}
=
\mathrm{PPL}(S)-\mathrm{PPL}(T),
\end{equation}
where lower is better; and the improvement over compressed initialization,
\begin{equation}
\Delta_{\mathrm{init}}
=
\mathrm{PPL}(S_{\mathrm{init}})-\mathrm{PPL}(S_{\mathrm{final}}),
\end{equation}
where higher is better. We also report the teacher-KD divergence under the corresponding KD objective, where lower is better.

\subsection{Geometry Advantage}

For each matched setting of model, initializer, rank, and seed, we compute
\begin{equation}
A_{\mathrm{geo}}(\mathcal{L}_{\mathrm{KD}})
=
\mathrm{PPL}(\mathcal{L}_{\mathrm{KD}})
-
\mathrm{PPL}(\mathcal{L}_{\mathrm{KD}}+\mathcal{L}_{\mathrm{geo}}).
\end{equation}
A positive value means that adding the geometry term improves final perplexity. The main comparisons are
\begin{equation}
A_{\mathrm{geo}}(\mathrm{RKL})
\qquad\text{and}\qquad
A_{\mathrm{geo}}(\mathrm{SkewKL}),
\end{equation}
because they test whether geometry is complementary to stronger autoregressive KD objectives rather than only to vanilla forward-KL distillation.

\subsection{Theoretical examination}
\label{app:distillation-theory}

We briefly justify the geometry term used in predictive-geometry distillation.
Let \(h^T_\ell\) and \(h^S_\ell\) denote teacher and student hidden states at
layer \(\ell\), and let
\[
\delta h_\ell = h^S_\ell-h^T_\ell .
\]
Let \(g_T\) be the teacher readout map used by the regularizer. In the full
version \(g_T\) is the downstream teacher map from layer \(\ell\) to the target
logits; in our implementation it is the readout-Fisher proxy described above.
Write
\[
z_T=g_T(h^T_\ell),
\qquad
z_S=g_T(h^S_\ell),
\qquad
\Delta z=z_S-z_T,
\]
and
\[
p_T=\mathrm{softmax}(z_T),
\qquad
p_S=\mathrm{softmax}(z_S).
\]

\subsubsection{Exact Fisher identity in logit space.}
Consider the straight path
\[
z(t)=z_T+t\Delta z,
\qquad
p(t)=\mathrm{softmax}(z(t)).
\]
Define
\[
\psi(t)
=
D_{\mathrm{KL}}
\left(
p_T
\middle\Vert
p(t)
\right).
\]
Then \(\psi(0)=0\), \(\psi'(0)=0\), and
\[
\psi''(t)
=
\Delta z^\top
F(p(t))
\Delta z,
\]
where
\[
F(p(t))=\mathrm{diag}(p(t))-p(t)p(t)^\top
\]
is the softmax Fisher matrix. Therefore, by the integral form of Taylor's
theorem,
\begin{equation}
D_{\mathrm{KL}}(p_T\Vert p_S)
=
\int_0^1
(1-t)
\Delta z^\top
F(p(t))
\Delta z
\,dt .
\label{eq:path-integrated-fisher}
\end{equation}
Thus the exact teacher-student KL in logit space is a path-integrated Fisher
energy.

\subsubsection{Local pullback to hidden space.}
If \(g_T\) is differentiable at \(h^T_\ell\), with Jacobian
\[
J_\ell = Dg_T(h^T_\ell),
\]
then
\[
\Delta z
=
J_\ell \delta h_\ell
+
O(\|\delta h_\ell\|^2).
\]
Using~\eqref{eq:path-integrated-fisher} and smoothness of \(F(p)\), we obtain
\begin{equation}
D_{\mathrm{KL}}(p_T\Vert p_S)
=
\frac12
\delta h_\ell^\top
J_\ell^\top F(p_T)J_\ell
\delta h_\ell
+
O(\|\delta h_\ell\|^3).
\label{eq:local-hidden-pullback}
\end{equation}
Hence the quadratic predictive-geometry term
\begin{equation}
\frac12
\delta h_\ell^\top
K^T_\ell
\delta h_\ell,
\qquad
K^T_\ell
=
J_\ell^\top F(p_T)J_\ell,
\end{equation}
is the local hidden-space pullback of the exact teacher-student KL induced by
the hidden-state mismatch. This is why the regularizer penalizes hidden errors
according to their effect on the teacher predictive distribution rather than
their Euclidean size.

\subsubsection{When a small geometry regularizer helps.}
Let
\[
A(\theta)
=
\mathcal{L}_{\mathrm{KD}}(\theta)
+
\lambda_{\mathrm{CE}}\mathcal{L}_{\mathrm{CE}}(\theta)
\]
be the base distillation objective, and let \(G(\theta)\) be the predictive-
geometry regularizer. Suppose \(\theta_0\) is a nondegenerate local minimizer of
\(A\), with Hessian
\[
H_A=\nabla^2 A(\theta_0)\succ 0
\]
on the trainable subspace. Let \(\theta_\lambda\) be the nearby local minimizer
of
\[
A(\theta)+\lambda G(\theta)
\]
for small \(\lambda>0\). By the implicit function theorem,
\begin{equation}
\theta_\lambda
=
\theta_0
-
\lambda
H_A^{-1}
\nabla G(\theta_0)
+
O(\lambda^2).
\label{eq:theta-lambda-expansion}
\end{equation}
Therefore, for any differentiable evaluation functional \(V(\theta)\),
\begin{equation}
V(\theta_\lambda)-V(\theta_0)
=
-
\lambda
\left\langle
\nabla V(\theta_0),
H_A^{-1}\nabla G(\theta_0)
\right\rangle
+
O(\lambda^2).
\label{eq:regularizer-help-condition}
\end{equation}
Thus, for sufficiently small \(\lambda\), the geometry term improves \(V\)
whenever
\begin{equation}
\left\langle
\nabla V(\theta_0),
H_A^{-1}\nabla G(\theta_0)
\right\rangle
>0.
\end{equation}

In particular, if \(V=G\), then
\begin{equation}
G(\theta_\lambda)-G(\theta_0)
=
-
\lambda
\nabla G(\theta_0)^\top
H_A^{-1}
\nabla G(\theta_0)
+
O(\lambda^2).
\end{equation}
Hence, whenever \(\nabla G(\theta_0)\neq 0\), a sufficiently small geometry
weight locally decreases the predictive-geometry mismatch. The regularizer is
therefore expected to help when the compressed student has hidden-state errors
with large Fisher-weighted energy that are not already removed by the base KD
objective. Conversely, if the base KD objective already controls these
prediction-relevant hidden directions, or if the student error lies mostly in
prediction-null directions, the geometry term may provide little benefit.

\begin{tcolorbox}[colback=gray!5!white, colframe=gray!40!black,
                  title={\small Informal summary: when geometry helps distillation},
                  left=4pt, right=4pt, top=3pt, bottom=3pt]
\small
The predictive-geometry regularizer is a hidden-state surrogate for teacher-student
predictive mismatch. In logit space, the exact teacher-student KL is a
path-integrated Fisher energy; locally, pulling this Fisher geometry back through
the teacher readout gives the quadratic penalty
\[
    \frac12\delta h_\ell^\top K^T_\ell\delta h_\ell .
\]
Thus the regularizer does not ask the student to match all hidden-state
directions equally. It mainly penalizes errors in directions that the teacher
uses for prediction.

Consequently, geometry regularization is expected to help when the compressed
student makes errors with large Fisher-weighted energy, i.e. errors aligned with
high-curvature directions of \(K^T_\ell\), and when the base KD objective has not
already corrected those directions. It may help little, or even hurt, when the
student errors are mostly prediction-null, when the base KD loss already controls
the relevant directions, or when the geometry weight is too large.
\end{tcolorbox}

\subsection{Distillation Results}
\label{app:pg-distillation-results}

Table~\ref{tab:kd-geom-adv-summary} reports the paired geometry advantage
\[
A_{\mathrm{geo}}(\mathcal{L}_{\mathrm{KD}})
=
\mathrm{PPL}(\mathcal{L}_{\mathrm{KD}})
-
\mathrm{PPL}(\mathcal{L}_{\mathrm{KD}}+\mathcal{L}_{\mathrm{geo}}),
\]
computed over matched completed settings. Positive values indicate that adding
the predictive-geometry term improves the compressed student at fixed model,
initializer, rank, and seed. The geometry term is not uniformly beneficial for
forward KL, but it improves stronger autoregressive KD objectives more reliably:
RKL+Geo improves over RKL in 9/14 matched settings, and SkewKL+Geo improves over
SkewKL in 10/14 matched settings. The largest average gain is obtained for
SkewKL, suggesting that predictive geometry is complementary to mode-seeking and
skewed KD objectives rather than simply replacing output-level distillation.

Because the sweep contains incomplete runs, we report paired statistics only on
matched completed settings. A complete listing of finished and unfinished runs is
provided in the supplementary tables.

\begin{table*}[h]
\centering
\tiny
\setlength{\tabcolsep}{3pt}
\renewcommand{\arraystretch}{0.88}
\caption{Completed KD-geometry sweep results for Qwen2.5-0.5B. Lower is better for PPL and KD divergence; higher is better for improvement over compressed initialization.}
\label{tab:kd-geom-full-results-qwen}
\resizebox{\textwidth}{!}{
\begin{tabular}{lllrrrr}
\toprule
Model & Method & Rank & Compression & Init PPL $\downarrow$ & Final PPL $\downarrow$ & $\Delta_{\mathrm{init}}\uparrow$ \\
\midrule
Qwen2.5-0.5B & asvd+fkl & 16 & 40.546x & 2.44e+05 & 581.4 & 2.43e+05 \\
Qwen2.5-0.5B & asvd+rkl & 16 & 40.546x & 2.44e+05 & 653.0 & 2.43e+05 \\
Qwen2.5-0.5B & asvd+skewkl & 16 & 40.546x & 2.44e+05 & 551.2 & 2.43e+05 \\
Qwen2.5-0.5B & svd+fkl & 16 & 40.546x & 2.15e+06 & 690.9 & 2.15e+06 \\
Qwen2.5-0.5B & svd+fkl\_geom & 16 & 40.546x & 2.15e+06 & 675.9 & 2.15e+06 \\
Qwen2.5-0.5B & svd+rkl & 16 & 40.546x & 2.15e+06 & 761.5 & 2.15e+06 \\
Qwen2.5-0.5B & svd+rkl\_geom & 16 & 40.546x & 2.15e+06 & 756.7 & 2.15e+06 \\
Qwen2.5-0.5B & svd+skewkl & 16 & 40.546x & 2.15e+06 & 690.9 & 2.15e+06 \\
Qwen2.5-0.5B & svd+skewkl\_geom & 16 & 40.546x & 2.15e+06 & 682.3 & 2.15e+06 \\
Qwen2.5-0.5B & asvd+fkl & 32 & 20.305x & 1.63e+06 & 590.6 & 1.63e+06 \\
Qwen2.5-0.5B & asvd+fkl\_geom & 32 & 20.305x & 1.63e+06 & 665.4 & 1.63e+06 \\
Qwen2.5-0.5B & asvd+rkl & 32 & 20.305x & 1.63e+06 & 598.1 & 1.63e+06 \\
Qwen2.5-0.5B & asvd+rkl\_geom & 32 & 20.305x & 1.63e+06 & 601.8 & 1.63e+06 \\
Qwen2.5-0.5B & asvd+skewkl & 32 & 20.305x & 1.63e+06 & 503.3 & 1.63e+06 \\
Qwen2.5-0.5B & asvd+skewkl\_geom & 32 & 20.305x & 1.63e+06 & 495.5 & 1.63e+06 \\
Qwen2.5-0.5B & svd+fkl & 32 & 20.305x & 77866 & 537.6 & 77328 \\
Qwen2.5-0.5B & svd+fkl\_geom & 32 & 20.305x & 77866 & 535.9 & 77330 \\
Qwen2.5-0.5B & svd+rkl & 32 & 20.305x & 77866 & 648.9 & 77217 \\
Qwen2.5-0.5B & svd+rkl\_geom & 32 & 20.305x & 77866 & 617.1 & 77249 \\
Qwen2.5-0.5B & svd+skewkl & 32 & 20.305x & 77866 & 556.4 & 77310 \\
Qwen2.5-0.5B & svd+skewkl\_geom & 32 & 20.305x & 77866 & 556.4 & 77310 \\
Qwen2.5-0.5B & asvd+fkl & 64 & 10.160x & 6.49e+05 & 324.4 & 6.49e+05 \\
Qwen2.5-0.5B & asvd+fkl\_geom & 64 & 10.160x & 6.49e+05 & 362.0 & 6.49e+05 \\
Qwen2.5-0.5B & asvd+rkl & 64 & 10.160x & 6.49e+05 & 418.2 & 6.49e+05 \\
Qwen2.5-0.5B & asvd+rkl\_geom & 64 & 10.160x & 6.49e+05 & 409.2 & 6.49e+05 \\
Qwen2.5-0.5B & asvd+skewkl & 64 & 10.160x & 6.49e+05 & 343.2 & 6.49e+05 \\
Qwen2.5-0.5B & asvd+skewkl\_geom & 64 & 10.160x & 6.49e+05 & 357.5 & 6.49e+05 \\
Qwen2.5-0.5B & svd+fkl & 64 & 10.160x & 1.53e+05 & 509.7 & 1.53e+05 \\
Qwen2.5-0.5B & svd+fkl\_geom & 64 & 10.160x & 1.53e+05 & 500.1 & 1.53e+05 \\
Qwen2.5-0.5B & svd+rkl & 64 & 10.160x & 1.53e+05 & 521.0 & 1.53e+05 \\
Qwen2.5-0.5B & svd+rkl\_geom & 64 & 10.160x & 1.53e+05 & 534.2 & 1.53e+05 \\
Qwen2.5-0.5B & svd+skewkl & 64 & 10.160x & 1.53e+05 & 462.4 & 1.53e+05 \\
Qwen2.5-0.5B & svd+skewkl\_geom & 64 & 10.160x & 1.53e+05 & 453.8 & 1.53e+05 \\
\bottomrule
\end{tabular}
}
\end{table*}

\begin{table*}[h]
\centering
\tiny
\setlength{\tabcolsep}{3pt}
\renewcommand{\arraystretch}{0.88}
\caption{Completed KD-geometry sweep results for SmolLM2 models. Lower is better for PPL and KD divergence; higher is better for improvement over compressed initialization.}
\label{tab:kd-geom-full-results-smollm}
\resizebox{\textwidth}{!}{
\begin{tabular}{lllrrrr}
\toprule
Model & Method & Rank & Compression & Init PPL $\downarrow$ & Final PPL $\downarrow$ & $\Delta_{\mathrm{init}}\uparrow$ \\
\midrule
SmolLM2-135M & asvd+fkl & 32 & 10.868x & 1.00e+08 & 713.0 & 1.00e+08 \\
SmolLM2-135M & asvd+fkl\_geom & 32 & 10.868x & 1.00e+08 & 706.3 & 1.00e+08 \\
SmolLM2-135M & asvd+rkl & 32 & 10.868x & 1.00e+08 & 831.4 & 1.00e+08 \\
SmolLM2-135M & asvd+rkl\_geom & 32 & 10.868x & 1.00e+08 & 808.3 & 1.00e+08 \\
SmolLM2-135M & asvd+skewkl & 32 & 10.868x & 1.00e+08 & 690.9 & 1.00e+08 \\
SmolLM2-135M & asvd+skewkl\_geom & 32 & 10.868x & 1.00e+08 & 693.1 & 1.00e+08 \\
SmolLM2-135M & svd+rkl\_geom & 32 & 10.868x & 4.85e+08 & 2171 & 4.85e+08 \\
SmolLM2-135M & svd+skewkl\_geom & 32 & 10.868x & 4.85e+08 & 1850 & 4.85e+08 \\
SmolLM2-135M & asvd+fkl & 64 & 5.434x & 8.94e+07 & 455.2 & 8.94e+07 \\
SmolLM2-135M & asvd+fkl\_geom & 64 & 5.434x & 8.94e+07 & 430.2 & 8.94e+07 \\
SmolLM2-135M & asvd+rkl & 64 & 5.434x & 8.94e+07 & 501.7 & 8.94e+07 \\
SmolLM2-135M & asvd+rkl\_geom & 64 & 5.434x & 8.94e+07 & 486.2 & 8.94e+07 \\
SmolLM2-135M & asvd+skewkl & 64 & 5.434x & 8.94e+07 & 423.5 & 8.94e+07 \\
SmolLM2-135M & asvd+skewkl\_geom & 64 & 5.434x & 8.94e+07 & 416.9 & 8.94e+07 \\
SmolLM2-135M & svd+fkl & 64 & 5.434x & 1.52e+08 & 1127 & 1.52e+08 \\
SmolLM2-135M & svd+fkl\_geom & 64 & 5.434x & 1.52e+08 & 1079 & 1.52e+08 \\
SmolLM2-135M & svd+rkl & 64 & 5.434x & 1.52e+08 & 1378 & 1.52e+08 \\
SmolLM2-135M & svd+rkl\_geom & 64 & 5.434x & 1.52e+08 & 1339 & 1.52e+08 \\
SmolLM2-135M & svd+skewkl & 64 & 5.434x & 1.52e+08 & 1323 & 1.52e+08 \\
SmolLM2-135M & svd+skewkl\_geom & 64 & 5.434x & 1.52e+08 & 1306 & 1.52e+08 \\
\midrule
SmolLM2-360M & asvd+fkl & 16 & 36.226x & 1.79e+05 & 1089 & 1.78e+05 \\
SmolLM2-360M & asvd+fkl\_geom & 16 & 36.226x & 1.79e+05 & 1254 & 1.78e+05 \\
SmolLM2-360M & asvd+rkl & 16 & 36.226x & 1.79e+05 & 1124 & 1.78e+05 \\
SmolLM2-360M & asvd+rkl\_geom & 16 & 36.226x & 1.79e+05 & 1082 & 1.78e+05 \\
SmolLM2-360M & asvd+skewkl & 16 & 36.226x & 1.79e+05 & 1219 & 1.78e+05 \\
SmolLM2-360M & asvd+skewkl\_geom & 16 & 36.226x & 1.79e+05 & 939.6 & 1.78e+05 \\
SmolLM2-360M & svd+fkl & 16 & 36.226x & 8.09e+07 & 1850 & 8.09e+07 \\
SmolLM2-360M & svd+fkl\_geom & 16 & 36.226x & 8.09e+07 & 2078 & 8.09e+07 \\
SmolLM2-360M & svd+rkl & 16 & 36.226x & 8.09e+07 & 2199 & 8.09e+07 \\
SmolLM2-360M & svd+rkl\_geom & 16 & 36.226x & 8.09e+07 & 2039 & 8.09e+07 \\
SmolLM2-360M & svd+skewkl & 16 & 36.226x & 8.09e+07 & 1805 & 8.09e+07 \\
SmolLM2-360M & svd+skewkl\_geom & 16 & 36.226x & 8.09e+07 & 1637 & 8.09e+07 \\
SmolLM2-360M & asvd+fkl & 32 & 18.113x & 1.46e+06 & 726.5 & 1.46e+06 \\
SmolLM2-360M & asvd+fkl\_geom & 32 & 18.113x & 1.46e+06 & 754.4 & 1.46e+06 \\
SmolLM2-360M & asvd+rkl & 32 & 18.113x & 1.46e+06 & 841.9 & 1.46e+06 \\
SmolLM2-360M & asvd+rkl\_geom & 32 & 18.113x & 1.46e+06 & 927.9 & 1.46e+06 \\
SmolLM2-360M & asvd+skewkl & 32 & 18.113x & 1.46e+06 & 747.3 & 1.46e+06 \\
SmolLM2-360M & asvd+skewkl\_geom & 32 & 18.113x & 1.46e+06 & 754.4 & 1.46e+06 \\
SmolLM2-360M & svd+fkl & 32 & 18.113x & 2.01e+07 & 1462 & 2.01e+07 \\
SmolLM2-360M & svd+fkl\_geom & 32 & 18.113x & 2.01e+07 & 1547 & 2.01e+07 \\
SmolLM2-360M & svd+rkl & 32 & 18.113x & 2.01e+07 & 2408 & 2.01e+07 \\
SmolLM2-360M & svd+rkl\_geom & 32 & 18.113x & 2.01e+07 & 2276 & 2.01e+07 \\
SmolLM2-360M & svd+skewkl & 32 & 18.113x & 2.01e+07 & 1880 & 2.01e+07 \\
SmolLM2-360M & svd+skewkl\_geom & 32 & 18.113x & 2.01e+07 & 1700 & 2.01e+07 \\
SmolLM2-360M & asvd+fkl & 64 & 9.057x & 78849 & 431.6 & 78418 \\
SmolLM2-360M & asvd+fkl\_geom & 64 & 9.057x & 78849 & 449.5 & 78400 \\
SmolLM2-360M & asvd+rkl & 64 & 9.057x & 78849 & 516.1 & 78333 \\
SmolLM2-360M & asvd+rkl\_geom & 64 & 9.057x & 78849 & 529.2 & 78320 \\
SmolLM2-360M & asvd+skewkl & 64 & 9.057x & 78849 & 461.0 & 78388 \\
SmolLM2-360M & asvd+skewkl\_geom & 64 & 9.057x & 78849 & 367.8 & 78482 \\
SmolLM2-360M & svd+fkl & 64 & 9.057x & 3.53e+07 & 1152 & 3.53e+07 \\
SmolLM2-360M & svd+fkl\_geom & 64 & 9.057x & 3.53e+07 & 1159 & 3.53e+07 \\
SmolLM2-360M & svd+rkl & 64 & 9.057x & 3.53e+07 & 1453 & 3.53e+07 \\
SmolLM2-360M & svd+rkl\_geom & 64 & 9.057x & 3.53e+07 & 1495 & 3.53e+07 \\
SmolLM2-360M & svd+skewkl & 64 & 9.057x & 3.53e+07 & 1193 & 3.53e+07 \\
SmolLM2-360M & svd+skewkl\_geom & 64 & 9.057x & 3.53e+07 & 1106 & 3.53e+07 \\
\bottomrule
\end{tabular}
}
\end{table*}

\begin{table}[!h]
\centering
\caption{Paired geometry comparisons for stronger KD baselines. Lower PPL is better; $\Delta_{\mathrm{Geo}}>0$ means that adding geometry improves final PPL.}
\label{tab:rkl-skewkl-geom-pairs}
\resizebox{\linewidth}{!}{
\begin{tabular}{llrrrrrrr}
\toprule
Model & Init. & Rank & RKL & RKL+Geo & $\Delta_{\mathrm{Geo}}\uparrow$ & SkewKL & SkewKL+Geo & $\Delta_{\mathrm{Geo}}\uparrow$ \\
\midrule
Qwen2.5-0.5B & ASVD & 32 & 598.1 & 601.8 & -3.8 & 503.3 & 495.5 & 7.8 \\
Qwen2.5-0.5B & ASVD & 64 & 418.2 & 409.2 & 9.1 & 343.2 & 357.5 & -14.3 \\
Qwen2.5-0.5B & SVD & 16 & 761.5 & 756.7 & 4.8 & 690.9 & 682.3 & 8.6 \\
Qwen2.5-0.5B & SVD & 32 & 648.9 & 617.1 & 31.8 & 556.4 & 556.4 & 0.0 \\
Qwen2.5-0.5B & SVD & 64 & 521.0 & 534.2 & -13.2 & 462.4 & 453.8 & 8.6 \\
SmolLM2-135M & ASVD & 32 & 831.4 & 808.3 & 23.1 & 690.9 & 693.1 & -2.2 \\
SmolLM2-135M & ASVD & 64 & 501.7 & 486.2 & 15.5 & 423.5 & 416.9 & 6.6 \\
SmolLM2-135M & SVD & 64 & 1378 & 1339 & 38.4 & 1323 & 1306 & 16.5 \\
SmolLM2-360M & ASVD & 16 & 1124 & 1082 & 41.5 & 1219 & 939.6 & 279.5 \\
SmolLM2-360M & ASVD & 32 & 841.9 & 927.9 & -86.0 & 747.3 & 754.4 & -7.1 \\
SmolLM2-360M & ASVD & 64 & 516.1 & 529.2 & -13.1 & 461.0 & 367.8 & 93.2 \\
SmolLM2-360M & SVD & 16 & 2199 & 2039 & 159.5 & 1805 & 1637 & 167.2 \\
SmolLM2-360M & SVD & 32 & 2408 & 2276 & 132.2 & 1880 & 1700 & 179.5 \\
SmolLM2-360M & SVD & 64 & 1453 & 1495 & -41.6 & 1193 & 1106 & 86.5 \\
\bottomrule
\end{tabular}
}
\end{table}

\clearpage
\newpage
\section{Additional figures.}

\subsection{Pruning results}

\begin{figure}[!h]
  \centering
  \begin{subfigure}{0.48\linewidth}
    \centering
    \includegraphics[width=0.9\linewidth]{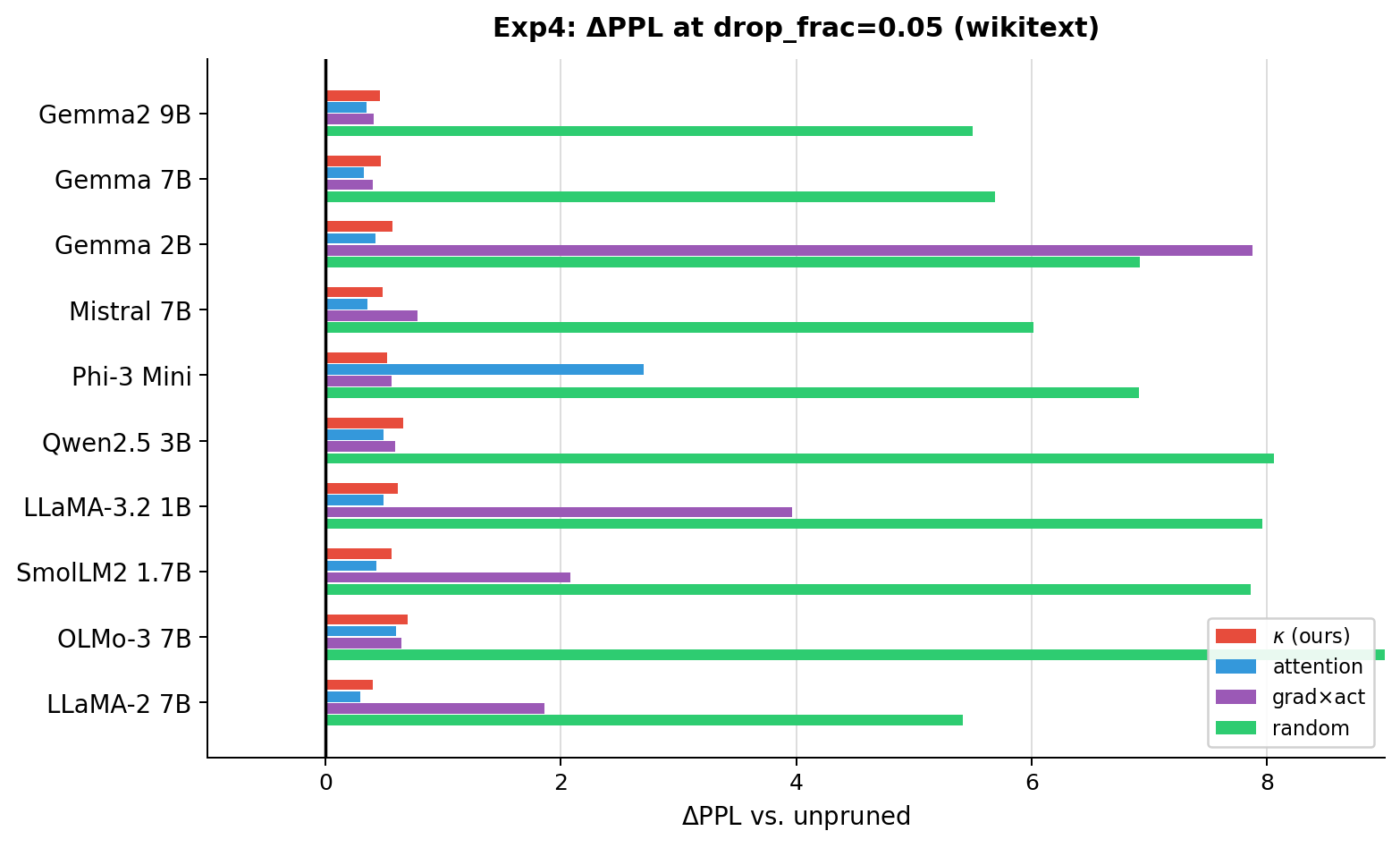}
    \label{fig:exp3_drop005}
  \end{subfigure}\hfill
  \begin{subfigure}{0.48\linewidth}
    \centering
    \includegraphics[width=0.9\linewidth]{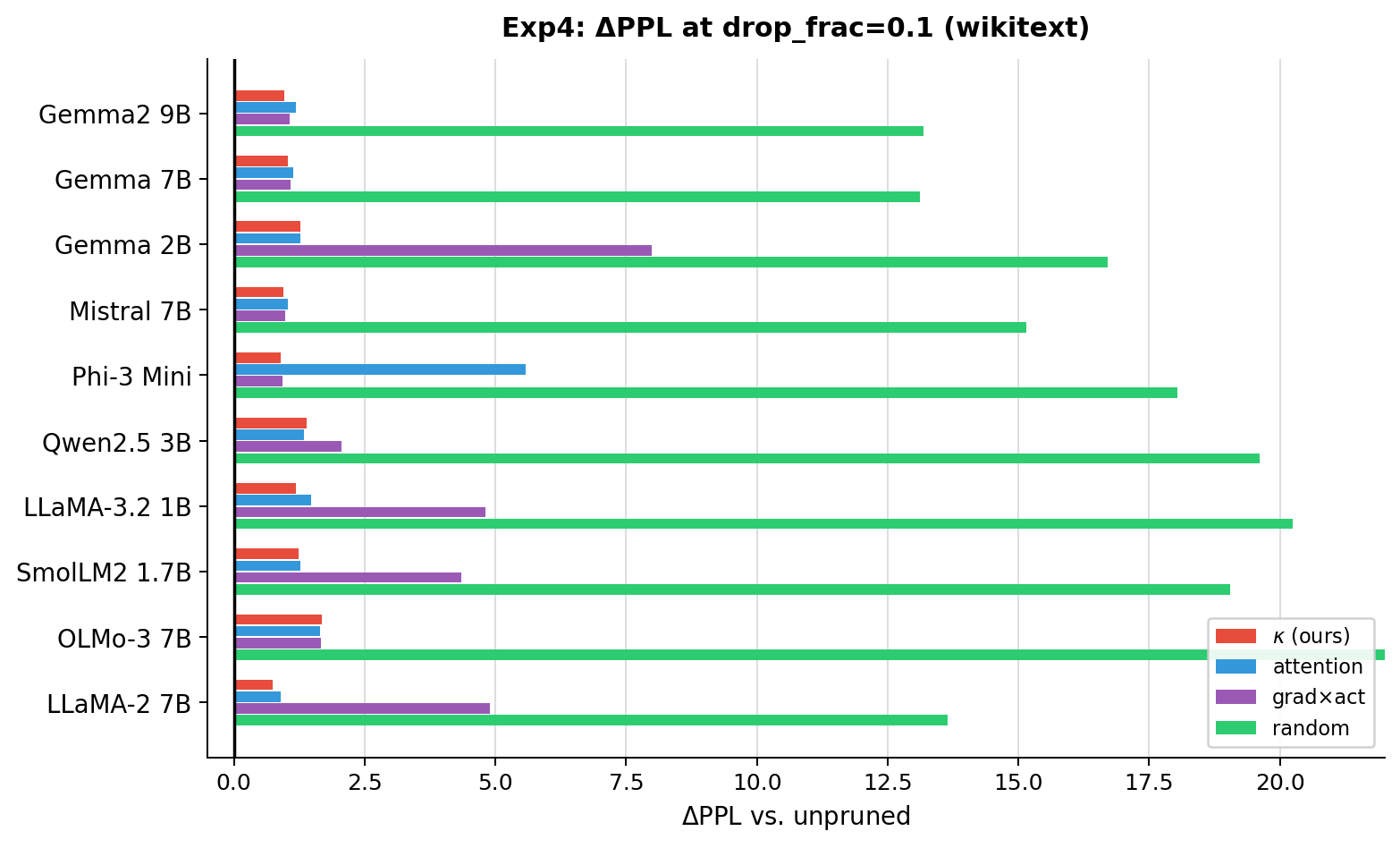}
    \label{fig:exp3_drop01}
  \end{subfigure}

  \vspace{0.5em}

  \begin{subfigure}{0.48\linewidth}
    \centering
    \includegraphics[width=0.9\linewidth]{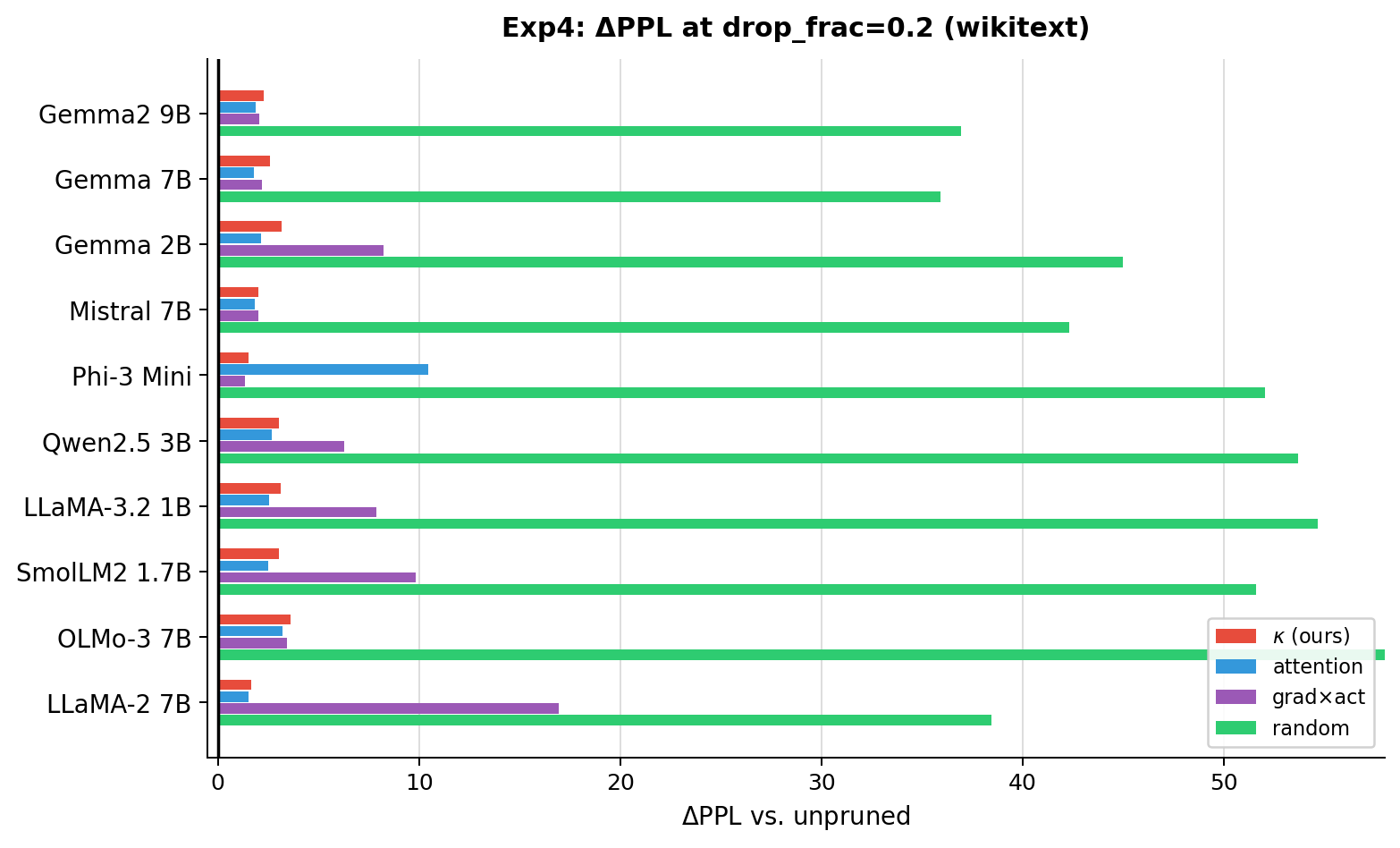}
    \label{fig:exp3_drop02}
  \end{subfigure}\hfill
  \begin{subfigure}{0.48\linewidth}
    \centering
    \includegraphics[width=0.9\linewidth]{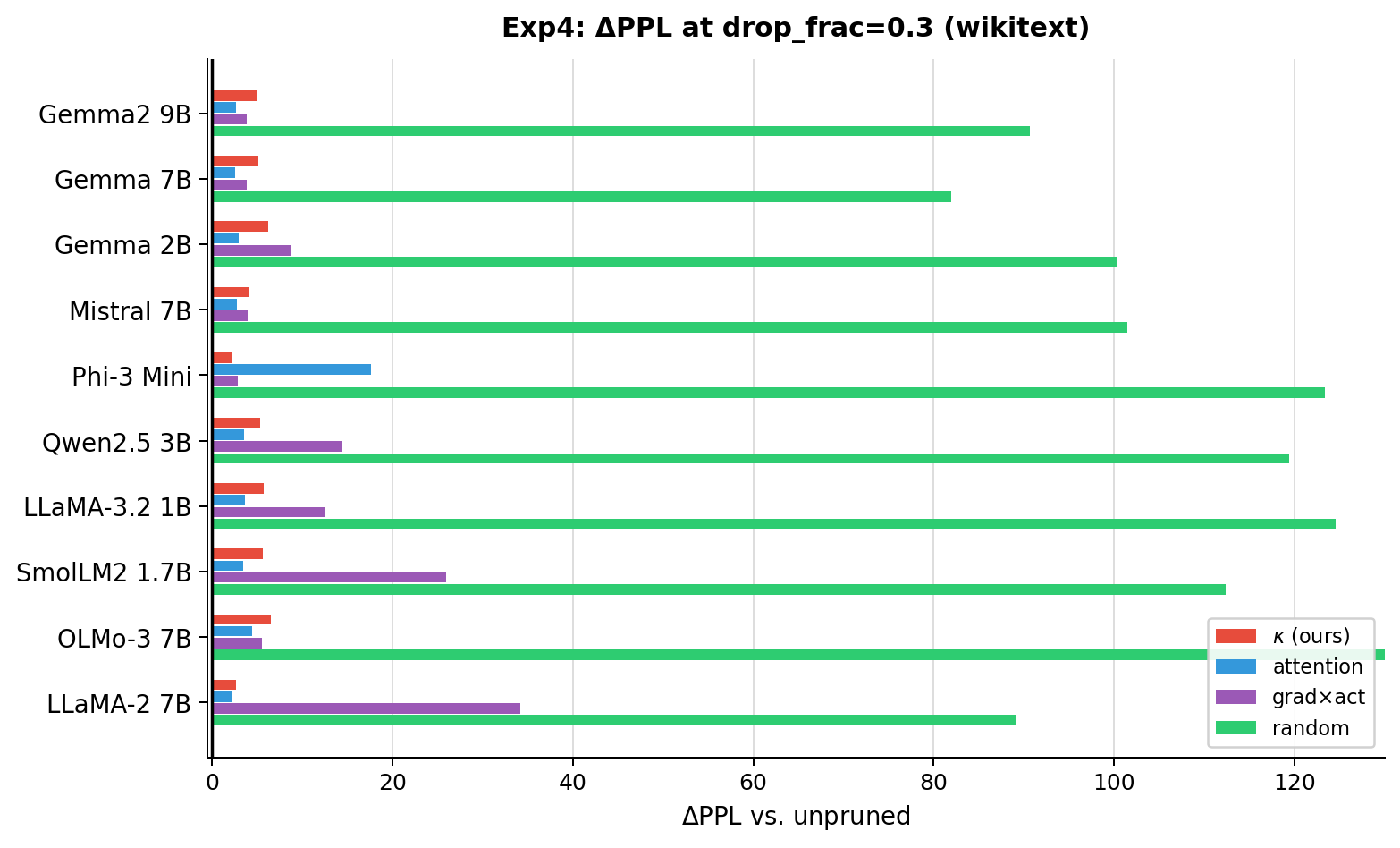}
    \label{fig:exp3_drop03}
  \end{subfigure}

  \caption{\textbf{Experiment 3: \(\Delta\mathrm{PPL}\) on \textsc{WikiText} under structured
  pruning for four drop fractions \(\rho \in \{0.05, 0.10, 0.20, 0.30\}\).}
  We compare our geometry-guided \texttt{kappa} criterion (red) to pruning based
  on attention scores (blue), a gradient-activation saliency baseline (purple), and random
  pruning (green); lower \(\Delta\mathrm{PPL}\) is better.}
  \label{fig:exp3_pruning_wikitext_all_drop}
\end{figure}

\newpage
\subsection{Direct validation of the quadratic approximation.}
\begin{figure}[!h]
  \centering
  \includegraphics[width=0.8\linewidth]{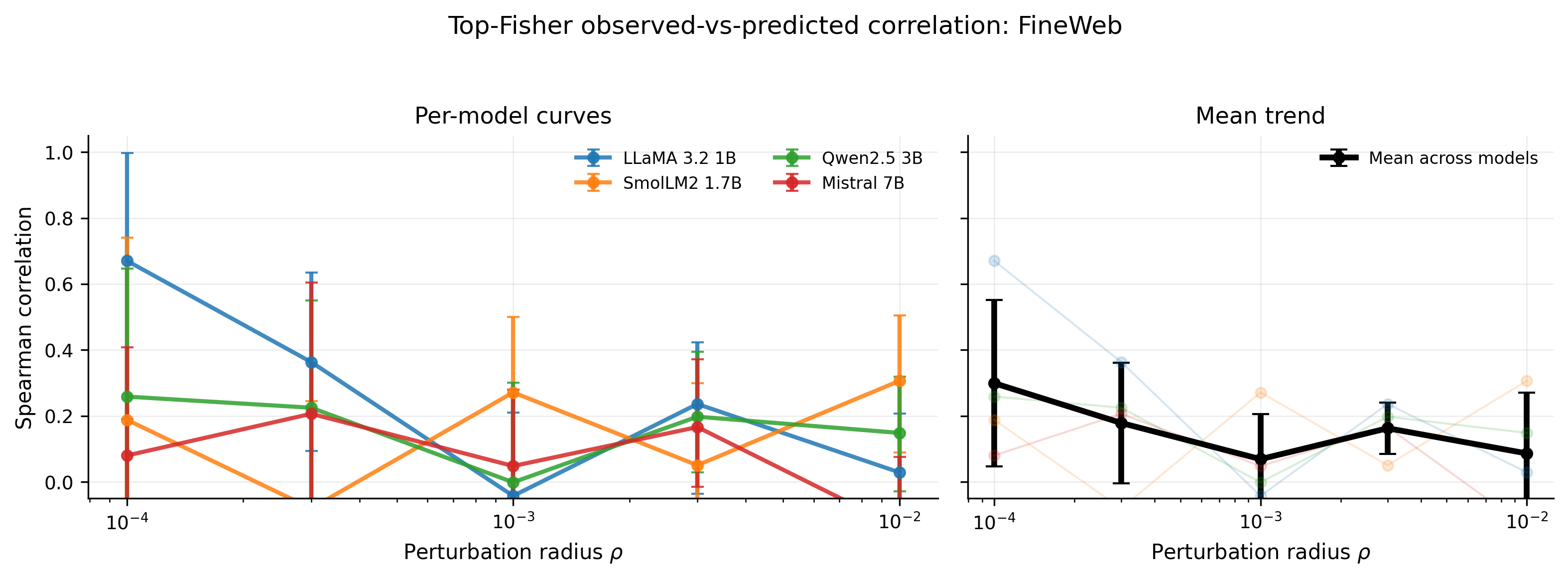}

  \vspace{0.5em}

  \includegraphics[width=0.8\linewidth]{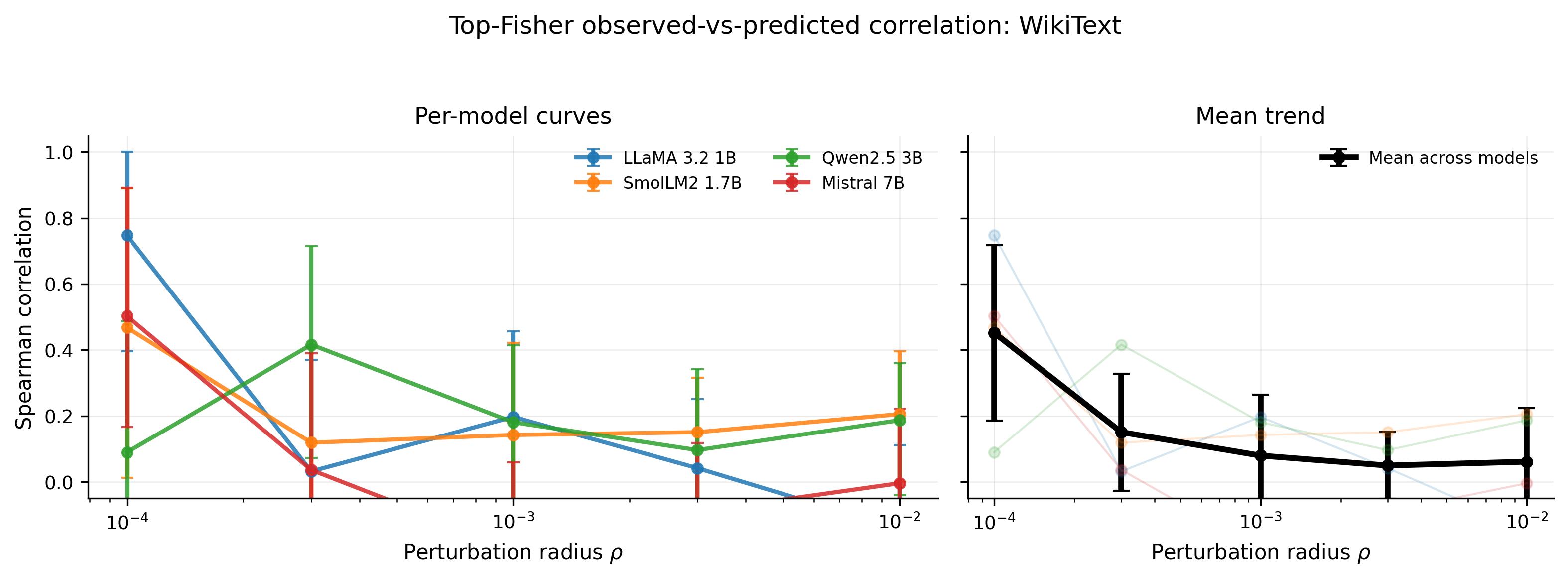}

  \caption{
  \textbf{Direct validation of the local pullback-Fisher approximation for top-Fisher
  perturbation directions on WikiText and FineWeb.} For each dataset, the left
  panel shows the Spearman correlation between the observed teacher-KL increase
  and the quadratic prediction
  \(\frac{1}{2}\alpha^2 v^\top K_\ell v\) for each model separately, as a
  function of the relative perturbation radius \(\rho\). The right panel shows
  the mean trend across models, with error bars denoting uncertainty across
  models. Correlations are highest at the smallest perturbation radii and
  generally decrease as \(\rho\) grows, consistent with the local nature of
  Theorem~\ref{th1}.
  }
  \label{fig:theory_just}
\end{figure}

\newpage
\subsection{Compute costs}
\begin{table}[!h]
\centering
\caption{Representative compute cost for geometry estimation on one dataset.
GPU hours are approximate wall-clock estimates for computing all reported
geometry statistics on the selected layer subset. Actual runtime depends on
GPU type, attention implementation, and model checkpoint.}
\label{tab:compute_cost}
\begin{tabular}{lrrrr}
\toprule
Model & Seq. len. & Examples & Model layers & GPU hours \\
\midrule
LLaMA-3.2-1B  & 512 & 32 & 16 & \(\approx 0.8\) \\
SmolLM2-1.7B  & 512 & 32 & 24 & \(\approx 1.0\) \\
Gemma-2B      & 512 & 32 & 26 & \(\approx 1.2\) \\
Qwen2.5-3B    & 512 & 32 & 36 & \(\approx 2.0\) \\
Phi-3 Mini    & 512 & 32 & 32 & \(\approx 2.2\) \\
LLaMA-2-7B    & 512 & 32 & 32 & \(\approx 4.0\) \\
Mistral-7B    & 512 & 32 & 32 & \(\approx 4.0\) \\
Gemma-7B      & 512 & 32 & 28 & \(\approx 4.5\) \\
OLMo-3-7B     & 512 & 32 & 32 & \(\approx 4.5\) \\
Gemma2-9B     & 512 & 32 & 42 & \(\approx 6.0\) \\
\bottomrule
\end{tabular}
\end{table}

\section{Limitations and Future Directions}

This work focuses on a local predictive-geometric description of hidden
trajectories in trained decoder-only transformers. We view the following points
as natural directions for extending the framework.

\paragraph{Local nature of the theory.}
Our theoretical results characterize layerwise loss-to-go functions near a
reference hidden trajectory \(X_\ell^\star\). This local viewpoint is deliberate:
it lets us connect the terminal prediction loss to measurable curvature and
observability structure inside the residual stream. A complementary direction is
to understand how these local geometries fit together across larger regions of
hidden-state space, and how such structure emerges during training.

\paragraph{Perturbation scale.}
The empirical Taylor validation depends on the perturbation radius. Very small
radii can make the measured loss changes close to numerical or sampling noise,
whereas very large radii can leave the local Taylor regime. This is why we report
performance across radii rather than at a single scale. More adaptive choices of
radius, possibly layer- or model-dependent, may give sharper local diagnostics.

\paragraph{Architectural scope.}
We study fixed trained decoder-only pre-norm transformers with causal
self-attention and a standard softmax readout. This setting already covers a
large class of language models and gives a clean causal interpretation of the
token-fiber score \(\kappa_{\ell,s}\). Extending the same predictive-geometric
view to encoder-decoder models, bidirectional transformers, mixture-of-experts
models, retrieval-augmented models, and architectures with explicit memory is an
interesting next step.

\paragraph{Computation of geometric quantities.}
The pullback-Fisher operator \(K_\ell\), its spectral summaries, and the
tokenwise scores \(\kappa_{\ell,s}\) can be estimated without materializing the
full matrix, using matrix-free JVP/VJP products. Nevertheless, these estimates
are more expensive than ordinary forward-pass statistics such as activation
norms or attention weights. Improving the efficiency of trace, spectrum, and
tokenwise-curvature estimation-for example through better probe sharing,
sketching, or amortized estimators-would make the approach more practical for
larger-scale deployment.

\paragraph{Reduction and pruning objectives.}
Our rank-allocation and token-pruning experiments are intended as controlled
tests of whether predictive geometry provides useful loss-aware signals. The
local theory suggests which directions or token updates are safer to discard
near a reference trajectory, but it is not meant to replace full compression
pipelines or task-specific pruning algorithms. In particular, direct
gradient-based heuristics can be highly competitive for a fixed pruning
intervention. A promising direction is to combine Fisher-geometric saliency with
such first-order criteria, rather than treating them as mutually exclusive.

\paragraph{Geometry-aware distillation.}
The hidden-state geometry term is designed to complement output-level
distillation by weighting student errors according to their predicted effect on
the teacher's output distribution. Our results suggest that this is most useful
with stronger autoregressive KD objectives such as reverse KL or skew KL.
Further work could study how the geometry term interacts with longer training
budgets, larger students, different low-rank parameterizations, and
instruction-tuned or chat models.

\newpage
\clearpage
\newpage
\section*{NeurIPS Paper Checklist}

\begin{enumerate}

\item {\bf Claims}
    \item[] Question: Do the main claims made in the abstract and introduction accurately reflect the paper's contributions and scope?

    \item[] Answer: \answerYes{}

    \item[] Justification: The abstract and introduction state the main theoretical and empirical contributions and explicitly describe the scope as local to neighborhoods of reference hidden-state trajectories. The claims are limited to fixed trained decoder-only transformers, local Fisher geometry, and the evaluated compression and pruning settings.

\item {\bf Limitations}
    \item[] Question: Does the paper discuss the limitations of the work performed by the authors?

    \item[] Answer: \answerYes{}

    \item[] Justification: The paper includes a dedicated Limitations section discussing the locality of the theory, architectural scope, computational cost of estimating curvature quantities, and the fact that the reduction guarantees are local rather than global compression guarantees.

\item {\bf Theory assumptions and proofs}
    \item[] Question: For each theoretical result, does the paper provide the full set of assumptions and a complete (and correct) proof?

    \item[] Answer: \answerYes{}

    \item[] Justification: The assumptions for Theorem~\ref{th1} and Propositions~\ref{pr1}--\ref{pr2} are stated in the main text, including smoothness, low-loss, causal-mask, and local-neighborhood assumptions. Full proofs are provided in Appendix~\ref{app:proofs}.

    \item {\bf Experimental result reproducibility}
    \item[] Question: Does the paper fully disclose all the information needed to reproduce the main experimental results of the paper to the extent that it affects the main claims and/or conclusions of the paper (regardless of whether the code and data are provided or not)?

    \item[] Answer: \answerYes{}

    \item[] Justification: The paper describes the model checkpoints, datasets, target-position sampling, calibration procedure, perturbation normalization, matrix-free curvature estimation, baselines, and evaluation metrics in the experimental section and Appendix~\ref{app:experimental_protocol}. The experiments use publicly available models and datasets, subject to their original access requirements.

\item {\bf Open access to data and code}
    \item[] Question: Does the paper provide open access to the data and code, with sufficient instructions to faithfully reproduce the main experimental results, as described in supplemental material?

    \item[] Answer: \answerYes{}

    \item[] Justification: The experiments use public datasets and public model checkpoints. We provide reproducibility instructions and code for matrix-free pullback-Fisher estimation, perturbation validation, rank-allocation experiments, pruning experiments, and baseline comparisons in the supplemental material and code archive.

\item {\bf Experimental setting/details}
    \item[] Question: Does the paper specify all the training and test details (e.g., data splits, hyperparameters, how they were chosen, type of optimizer) necessary to understand the results?

    \item[] Answer: \answerYes{}

    \item[] Justification: The work evaluates fixed pretrained models and does not train new models. The paper specifies the datasets, splits, context length, calibration and evaluation sampling, target positions, perturbation radii, rank thresholds, pruning fractions, numerical precision, and baseline settings in Section~\ref{sec:experiments} and Appendix~\ref{app:experimental_protocol}.

\item {\bf Experiment statistical significance}
    \item[] Question: Does the paper report error bars suitably and correctly defined or other appropriate information about the statistical significance of the experiments?

    \item[] Answer: \answerYes{}

    \item[] Justification: The main experimental results report bootstrap confidence intervals or standard errors over validation examples, perturbation directions, and randomized numerical estimates where applicable. The paper states how error bars are computed and uses nonparametric bootstrap intervals rather than assuming normality.

\item {\bf Experiments compute resources}
    \item[] Question: For each experiment, does the paper provide sufficient information on the computer resources (type of compute workers, memory, time of execution) needed to reproduce the experiments?

    \item[] Answer: \answerYes{}

    \item[] Justification: The paper reports representative wall-clock time, peak GPU memory, number of calibration tokens, number of curvature products, and hardware used for the main experiments in Table~\ref{tab:compute_cost}. We also describe the scaling of the matrix-free estimation cost with examples, layers, and probe directions.
    
\item {\bf Code of ethics}
    \item[] Question: Does the research conducted in the paper conform, in every respect, with the NeurIPS Code of Ethics \url{https://neurips.cc/public/EthicsGuidelines}?

    \item[] Answer: \answerYes{}

    \item[] Justification: The work analyzes and compresses existing public language models using public text datasets and does not collect private data, involve human subjects, or deploy a system that makes decisions about individuals.

\item {\bf Broader impacts}
    \item[] Question: Does the paper discuss both potential positive societal impacts and negative societal impacts of the work performed?

    \item[] Answer: \answerYes{}

    \item[] Justification: The paper discusses that geometry-guided compression and pruning may reduce the computational cost and energy use of language-model inference. It also notes the dual-use risk that cheaper inference can make both beneficial and harmful uses of language models more accessible.
    
\item {\bf Safeguards}
    \item[] Question: Does the paper describe safeguards that have been put in place for responsible release of data or models that have a high risk for misuse (e.g., pre-trained language models, image generators, or scraped datasets)?

    \item[] Answer: \answerNA{}

    \item[] Justification: The paper does not release a new pretrained language model, image generator, or scraped dataset. It releases analysis and compression code for existing models, whose access conditions and safeguards are governed by the original model providers.

\item {\bf Licenses for existing assets}
    \item[] Question: Are the creators or original owners of assets (e.g., code, data, models), used in the paper, properly credited and are the license and terms of use explicitly mentioned and properly respected?

    \item[] Answer: \answerYes{}

    \item[] Justification: The paper identifies and cites the existing model checkpoints, datasets, and software libraries used in the experiments. We follow the access restrictions and license terms of the original model and dataset providers, including gated-access requirements where applicable.

\item {\bf New assets}
    \item[] Question: Are new assets introduced in the paper well documented and is the documentation provided alongside the assets?

    \item[] Answer: \answerYes{}

    \item[] Justification: The new asset introduced by the paper is code for estimating the proposed geometric quantities and reproducing the experiments. The released code includes documentation, commands, configuration options, and descriptions of the expected outputs; no new dataset or pretrained model checkpoint is introduced.

\item {\bf Crowdsourcing and research with human subjects}
    \item[] Question: For crowdsourcing experiments and research with human subjects, does the paper include the full text of instructions given to participants and screenshots, if applicable, as well as details about compensation (if any)? 

    \item[] Answer: \answerNA{}

    \item[] Justification: The paper does not involve crowdsourcing, user studies, annotation tasks, or research with human subjects.

\item {\bf Institutional review board (IRB) approvals or equivalent for research with human subjects}
    \item[] Question: Does the paper describe potential risks incurred by study participants, whether such risks were disclosed to the subjects, and whether Institutional Review Board (IRB) approvals (or an equivalent approval/review based on the requirements of your country or institution) were obtained?

    \item[] Answer: \answerNA{}

    \item[] Justification: The paper does not involve human-subjects research, crowdsourcing, or collection of new data from individuals, so IRB approval or equivalent review is not applicable.
    
\item {\bf Declaration of LLM usage}
   \item[] Question: Does the paper describe the usage of LLMs if it is an important, original, or non-standard component of the core methods in this research? Note that if the LLM is used only for writing, editing, or formatting purposes and does \emph{not} impact the core methodology, scientific rigor, or originality of the research, declaration is not required.

    \item[] Answer: \answerNA{}

    \item[] Justification: The core methodology does not use an LLM as an original or non-standard research component.  

\end{enumerate}

\end{document}